\documentclass{article} 
\usepackage{iclr2027_conference,times}

\usepackage{amsmath,amsfonts,bm}

\def\eqref#1{equation~\ref{#1}}

\def\1{\bm{1}}

\DeclareMathAlphabet{\mathsfit}{\encodingdefault}{\sfdefault}{m}{sl}
\SetMathAlphabet{\mathsfit}{bold}{\encodingdefault}{\sfdefault}{bx}{n}

\usepackage{hyperref}
\usepackage{url}

\usepackage{booktabs}       
\usepackage{amsfonts}       
\usepackage{nicefrac}       
\usepackage{microtype}      
\usepackage{xcolor}         

\usepackage{makecell}
\usepackage{multirow}
\usepackage[table]{xcolor}
\usepackage{wrapfig}
\usepackage{adjustbox}
\usepackage{enumitem}

\usepackage{amssymb} 
\usepackage{amsthm}
\newtheorem{proposition}{Proposition}
\usepackage{amsmath}
\usepackage{tikz}
\usetikzlibrary{arrows.meta, positioning, shapes.multipart}
\usepackage[skip=4pt]{caption}

\usepackage{cleveref}

\crefformat{section}{\S#2#1#3}
\Crefformat{section}{\S#2#1#3}
\crefformat{subsection}{\S#2#1#3}
\Crefformat{subsection}{\S#2#1#3}
\crefformat{subsubsection}{\S#2#1#3}
\Crefformat{subsubsection}{\S#2#1#3}

\crefname{figure}{Fig.}{Figs.}
\Crefname{figure}{Fig.}{Figs.}
\crefname{table}{Tab.}{Tabs.}
\Crefname{table}{Tab.}{Tabs.}
\crefname{equation}{Eq.}{Eqs.}
\Crefname{equation}{Eq.}{Eqs.}
\newcommand{\appxref}[1]{Appdx.~\ref{#1}}

\usepackage{xspace}  
\newcommand{\modelname}{\modelnamens\xspace}
\newcommand{\modelnamens}{NARA}
\newcommand{\maskedmodeling}{\maskedmodelingns\xspace}
\newcommand{\maskedmodelingns}{MGSM}
\newcommand{\georelmodeling}{\georelmodelingns\xspace}
\newcommand{\georelmodelingns}{GSRM}

\title{\modelname: Anchor-Conditioned Representation Learning for Heterogeneous Vector Geoentities}

\author{
Jina Kim\textsuperscript{1} \quad 
Gengchen Mai\textsuperscript{2} \quad
Lingyi Zhao\textsuperscript{3} \quad
Khurram Shafique\textsuperscript{3} \quad
Yao-Yi Chiang\textsuperscript{1} \\
\\
\textsuperscript{1}Department of Computer Science and Engineering, University of Minnesota, Minnesota, USA \\
\textsuperscript{2}Department of Geography and the Environment, University of Texas at Austin, Texas, USA \\
\textsuperscript{3}Novateur Research Solutions, Virginia, USA \\
\\
\texttt{\{kim01479, yaoyi\}@umn.edu}, \texttt{gengchen.mai@austin.utexas.edu}, \\ \texttt{\{lzhao, kshafique\}@novateur.ai} \\
}

\iclrfinalcopy 
\begin{document}
\maketitle

\begin{abstract}
Vector geospatial data represent the world as discrete geoentities, such as roads, buildings, and points of interest (POIs), each with semantic attributes, geometry, and spatial relations to other geoentities, including metric proximity and topology. Existing methods for learning geoentity representations typically support a single geometry type or model only a subset of these relations, limiting their ability to capture spatial context across heterogeneous geoentities and support diverse downstream tasks. We propose \textbf{\modelname} (\textbf{N}eural \textbf{A}nchor-conditioned \textbf{R}elation-\textbf{A}ware representation learning), a novel self-supervised representation framework for heterogeneous vector geoentities. \modelname contextualizes geoentities through spatial-context-aware attention that models spatial autocorrelation using geometry distance modulated by topological relations across surrounding points, polylines, and polygons. \modelname introduces masked geoentity semantic modeling and geometry-aware spatial relation modeling, as well as relation-conditioned regularization that encourages similar representations for geoentities sharing the same spatial relation to a common reference entity, while accounting for spatial autocorrelation. \modelname's frozen, task-agnostic encoder outperforms state-of-the-art methods, each with an architecture tailored to its respective task, across traffic-speed prediction for polylines, building-function classification for polygons, and next point-of-interest prediction for points. 
\end{abstract}


\vspace{-2em}
\section{Introduction}
\vspace{-1em}
Vector maps such as OpenStreetMap (OSM)~\citep{osm} and Overture Maps~\citep{overturemaps} provide large-scale, human-curated abstractions of the world as collections of discrete \textit{geoentities}, including roads, buildings, and points of interest (POIs). Unlike sensed raster data (e.g., satellite imagery), vector data reflect intentional human curation~\citep{goodchild1992geographical}, in which important entities are selected, preserved, and represented at different geometric granularities (e.g., polygons vs.\ points) according to their functional roles~\citep{longley2015geographic}. Geoentities are further organized through spatial relations, both metric and topological, that characterize their spatial configurations within the surrounding environment. Vector data thus describe what geoentities are, what shapes they have, and how they are spatially related, providing a rich foundation for representation learning that supports diverse geospatial tasks~\citep{mai2024srl,siampou2025polyvec}.

Learning task-agnostic representations across geometry types requires spatial contextualization~\citep{balsebre2024city}, in which a target geoentity is represented through its semantics, geometry, and spatial relations to surrounding geoentities~\citep{mai2020multi}. For example, a Starbucks \textit{inside a mall} may exhibit visiting patterns unlike one \textit{in an airport} or \textit{along Fifth Avenue}, and a Starbucks mapped as a polygon rather than a point may denote a larger venue. Existing methods on contextual geoentity representation learning, however, often support only a single geometry type, learning representations separately for points~\citep{mai2020multi,li-etal-2022-spabert,cheng2025poi}, polylines~\citep{zhang2023road,zhou2024road}, or polygons~\citep{li2023urban,mai2023towards} with type-specific architectures that do not transfer across geometry types. Recent methods supporting multiple geometry types model spatial context only partially, through point-based proximity~\citep{tempelmeier2021geovectors} or predefined relational schemas that depend on sparse auxiliary data such as mobility~\citep{yang2025hygmap} or transit routes~\citep{balsebre2024city}. A distinct line of work learns unified geometry representations across points, polylines, and polygons~\citep{siampou2025polyvec,chu2026geo2vec}. These methods encode geometry effectively but represent each geoentity in isolation, without modeling their semantics or spatial context. To our knowledge, no existing approach jointly contextualizes semantics, complete geometry, and spatial relations across heterogeneous geoentities, the capability needed for task-agnostic representations that transfer across geometry types and downstream tasks.

We propose \textbf{\modelname} (\textbf{N}eural \textbf{A}nchor-conditioned \textbf{R}elation-\textbf{A}ware representation learning), a novel self-supervised framework for learning task-agnostic contextual representations of heterogeneous vector geoentities. \modelname contextualizes each geoentity from its semantics, complete geometry, and spatial relations to surrounding points, polylines, and polygons through spatial-context-aware attention, whose topology-modulated distance bias modulates the effect of complete-geometry distance by topological relation. \modelname introduces (i) \emph{masked geoentity semantic modeling}, which reconstructs a target's masked semantics from its retained geometry and spatial context, (ii) \emph{geometry-aware spatial relation modeling}, which requires the contextual representations to encode spatial relations, and (iii) \emph{relation-conditioned regularization}, which encourages geoentities sharing the same spatial relation to a common reference entity, termed an \emph{anchor}, to learn similar representations beyond what their proximity explains. We pretrain \modelname separately on New York City (NYC) and Singapore with the same architecture and keep its task-agnostic encoder frozen for downstream tasks. \modelname outperforms state-of-the-art approaches with task-specific architectures and task-aligned pretraining objectives on all three tasks spanning every geometry type.

\vspace{-1em}
\section{Related Work}
\vspace{-1em}
Existing representation learning methods for geospatial vector data are typically designed for a single geometry type (e.g., points)~\citep{klemmer2025satclip,liu2026gair}, rely on specific data sources or attributes (e.g., place categories)~\citep{balsebre2024city}, and define context using fixed similarity notions such as spatial proximity or predefined functional relations~\citep{chen2025self}. This limits their ability to capture spatial interactions and generalize across heterogeneous geoentities.
Point-based methods encode semantic attributes with coordinate embeddings~\citep{mai2020multi,li-etal-2022-spabert}, but lack spatial extent and topology. Line-based methods, largely focused on road networks, learn from graph connectivity~\citep{grover2016node2vec,wang2019learning,jepsen2020relational,zhang2023road,zhou2024road} but ignore geometric shape and spatial relations with surrounding entities, and are difficult to generalize beyond road networks. Polygon-based methods either rasterize shapes for visual encoders~\citep{mai2023towards,balsebre2024city,li2023urban}, losing geometric precision and scale, or encode coordinates only (e.g.,~\cite{mai2023towards}), and model entities in isolation without inter-entity context. Geometry-type-agnostic encoders~\citep{siampou2025polyvec,chu2026geo2vec} unify representations across geometry types but focus solely on geometry and require downstream adaptation to capture semantic, metric, or topological relationships. More general frameworks in urban computing~\citep{tempelmeier2021geovectors,balsebre2024city,yang2025hygmap} contextualize heterogeneous entities, yet rely on proximity or predefined functional dependencies from available data schemas (e.g., OSM relations), rather than learning how topological relations shape functional similarity. As a result, semantics and spatial structure are treated separately instead of jointly as complementary information sources. \modelname addresses these limitations by jointly encoding semantic attributes and geometry across heterogeneous entity types. In addition, prior work has incorporated spatial autocorrelation, one of the critical spatial principles, through strategies such as training data sampling~\citep{yan2017itdl}, negative sampling~\citep{mai2019relaxing,huang2022estimating}, and specialized loss functions~\citep{jean2019tile2vec,wang2024mc}. However, these approaches focus solely on metric relations and are limited to satellite imagery~\citep{jean2019tile2vec,li2023rethink} or point-based entities~\citep{mai2019relaxing,huang2022estimating,wang2024mc}. In contrast, \modelname models spatial autocorrelation across geoentities with diverse geometry types and explicitly captures how metric and topological relations jointly shape contextual representations of vector geoentities.

\vspace{-1em}
\section{Spatial Contextualization Principles}\label{sec:theory}
\vspace{-1em}

\textbf{Preliminaries}
Given a set of geoentities $\mathcal{G}=\{v_1,\dots,v_n\}$, each entity $v_i$ has semantic attributes $s_i$ (e.g., a category or textual description), geometry $g_i$, and geometry type $\eta(v_i)\in\{\mathrm{point},\mathrm{polyline},\mathrm{polygon}\}$. Our goal is to learn a function $f:\mathcal{G}\rightarrow\mathbb{R}^{d_{\mathrm{model}}}$ that maps these geoentities to a shared space while retaining their semantics, geometry, and spatial structure, including both metric proximity and topological relations.

\textbf{Spatial Dependencies}\label{sec:theory_spatial_dependencies}
Geographic systems exhibit structured spatial dependencies. Tobler's First Law states that nearby entities tend to be more related than distant ones~\citep{tobler1970computer,tobler2004first}, motivating \emph{spatial autocorrelation}, where similarity decays with distance, often characterized by the semivariogram~\citep{bachmaier2011variogram}. \citet{lin2020deep,wang2024mc} further generalized the principle of semivariograms into the latent embedding space:
\begin{equation}
  \gamma(b) \;=\; \frac{1}{2}\,\mathbb{E}\!\left[\|Z(v_i)-Z(v_j)\|^2 \,\middle|\, d_{ij}=b\right],
  \label{eq:variogram_def}
\end{equation}
where $Z(v_i)$ denotes a latent representation associated with geoentity $v_i$, and $d_{ij}$ is the geographic distance between $v_i$ and $v_j$. Smaller $\gamma(b)$ indicates stronger similarity at distance separation $b$, and $\gamma(b)$ typically increases with distance in a nonlinear, saturating manner. For vector geoentities, this suggests that nearby entities provide stronger contextual information and that contextual influence should decay with distance in a data-dependent way. 

Beyond metric proximity, spatial reasoning also depends on \emph{topological} structure (e.g., containment and adjacency)~\citep{egenhofer1990mathematical,randell1992spatial,regalia2019computing}. Such relations can override proximity: entities sharing meaningful topology (e.g., stores within the same mall or businesses along the same avenue) may exhibit similar functions despite being geographically separated. Let $a$ denote a reference (anchor) entity and $\text{rel}(v_i, a) \in \mathcal{R}$ the topological relation between $v_i$ and $a$. We expect entities sharing the same relation to a common anchor to be more similar than entities at the same distance in general:
\begin{equation}
\mathbb{E}\!\left[
\operatorname{sim}\!\big(Z(v_i),Z(v_j)\big)
\,\middle|\,
\mathrm{rel}(v_i,a)=\mathrm{rel}(v_j,a)\neq\varnothing,\;
d_{ij}=b
\right]
>
\mathbb{E}\!\left[
\operatorname{sim}\!\big(Z(v_i),Z(v_j)\big)
\,\middle|\,
d_{ij}=b
\right].
\label{eq:relational_autocorr}
\end{equation}
Here, \Cref{eq:relational_autocorr} should be viewed as a modeling assumption motivated by geographic spatial organization rather than a theoretical guarantee. Therefore, motivated by \Cref{eq:variogram_def} and \Cref{eq:relational_autocorr}, we model functional similarity as jointly governed by metric distance and relational agreement:
\begin{equation}
\operatorname{sim}(Z(v_i), Z(v_j))
\propto
k(d_{ij})
\cdot
\mathbb{E}_{a \sim p(a \mid v_i,v_j)}
\!\left[
\psi\big(\mathrm{rel}(v_i,a), \mathrm{rel}(v_j,a)\big)
\right],
\label{eq:topo_modulated_kernel}
\end{equation}
where $k(\cdot)$ models distance decay and $\psi(\cdot)$ captures topological modulation. The expectation is taken over anchors relevant to the pair $(v_i,v_j)$, allowing entities to participate in multiple relational contexts.~\Cref{eq:topo_modulated_kernel} is not intended as a probabilistic generative model, but rather as a conceptual formulation motivating the representation objectives introduced in~\Cref{sec:method}. This perspective suggests two desirable properties for contextual representation learning over heterogeneous geoentities. \textbf{Distance-decay:} contextual similarity decreases with geographic distance in a nonlinear manner. \textbf{Topological modulation:} entities sharing similar topological relations to a common anchor should exhibit enhanced similarity beyond metric proximity alone. 
\vspace{-1em}
\section{\modelname}\label{sec:method}
\vspace{-1em}
\Cref{sec:theory} motivates the design of \modelname (\Cref{fig:architecture}), a self-supervised framework for learning contextual geoentity representations, in which pairwise metric and topological relations directly modulate contextualization through topology-modulated distance bias, and relation-conditioned regularization regularizes anchor relation-dependent similarity beyond pairwise proximity.
\begin{figure}[h]
 \vspace{-.5em}
  \centering
  \includegraphics[width=\textwidth]{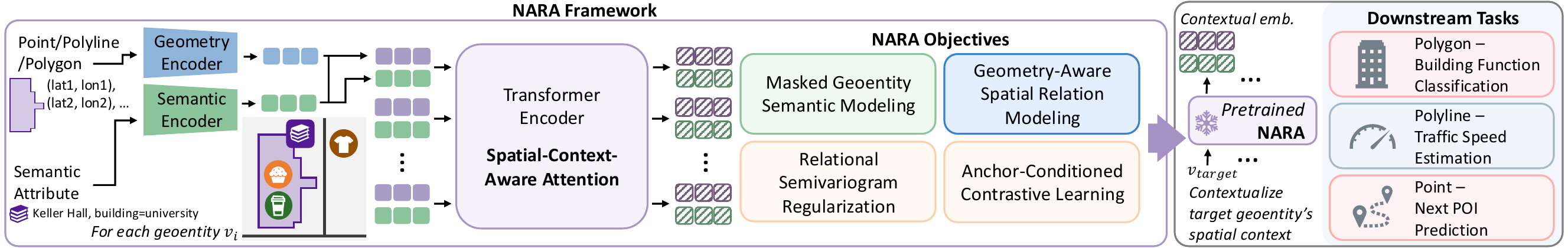}
  \caption{Overview of \modelname, the proposed self-supervised learning framework.}
  \label{fig:architecture}
\vspace{-1.2em}
\end{figure}
\vspace{-1em}
\subsection{Semantic and Geometric Representations of Geoentity}\label{sec:geo_encoding}\vspace{-.5em}
For each geoentity $v_i\in\mathcal{G}$, \modelname encodes its semantic attributes $s_i$, represented as a sentence (e.g., composed of name, category, or other textual descriptors), with a language encoder (e.g., BERT~\citep{devlin-etal-2019-bert}), and its complete geometry $g_i$ with a geometry encoder (e.g., Poly2Vec~\citep{siampou2025polyvec}), to preserve shape, spatial extent, and spatial relations across geometry types:
\begin{equation}
\mathbf{e}_i^{\mathrm{sem}}
=
\operatorname{Enc}_{\mathrm{sem}}(s_i),
\qquad
\mathbf{e}_i^{\mathrm{geom}}
=
\operatorname{Enc}_{\mathrm{geom}}(g_i).
\vspace{-.2em}
\end{equation}
\modelname projects both embeddings into $\mathbb{R}^{d_{\mathrm{model}}}$, yielding $\mathbf{z}_i^{\mathrm{sem}}$ and $\mathbf{z}_i^{\mathrm{geom}}$, and fuses them adaptively:
\vspace{-.5em}
\begin{equation}
\rho_i
=
\sigma\!\left(
f_{\mathrm{gate}}
\bigl([\mathbf{z}_i^{\mathrm{sem}};\mathbf{z}_i^{\mathrm{geom}}]\bigr)
\right),
\qquad
\mathbf{z}_i^{\mathrm{fuse}}
=
(1-\rho_i)\mathbf{z}_i^{\mathrm{sem}}
+
\rho_i\mathbf{z}_i^{\mathrm{geom}}.
\label{eq:fused_emb}
\vspace{-.2em}
\end{equation}
The scalar $\rho_i\in[0,1]$ modulates the contribution of geometric information. We initialize the gating bias negatively so that training prioritizes semantic information and gradually incorporates geometry.
\vspace{-.5em}
\subsection{Masked Geoentity Semantic Modeling with Spatial-Context Attention}\label{sec:mgsm}
\vspace{-.5em}
Following~\Cref{sec:theory}, \modelname contextualizes each geoentity through its semantics, geometry, and metric and topological relations to surrounding geoentities. We propose \textbf{Masked Geoentity Semantic Modeling (\maskedmodeling) with Spatial-Context-Aware attention}. Given a target whose semantic representation is masked, \maskedmodeling reconstructs its semantics from its retained geometry and the surrounding heterogeneous geoentities. For example, a point geoentity with masked semantics is likely a ``shop'' if it lies within a ``commercial building'' polygon, but more likely a ``parking lot'' if it lies just outside the same polygon. \modelname models these interactions using a Transformer encoder over local geoentities with spatial-context-aware attention, which determines how much information a target geoentity receives from each neighbor based on the neighbor's semantics, geometry, and spatial relation to the target.

We define local spatial contexts by covering the target area (e.g., NYC) with overlapping windows $\{\mathcal{W}_\ell\}_{\ell=1}^{L}$ (e.g., $250\,\mathrm{m}\times250\,\mathrm{m}$) placed at a fixed stride. Each window yields a context $\mathcal{N}(\ell)$, the set of heterogeneous geoentities intersecting $\mathcal{W}_\ell$, and the overlap preserves relational neighborhoods across window boundaries. Because geoentities near a window boundary may have incomplete context, \modelname samples masked targets only from a central region of each window, allowing all entities in $\mathcal{N}(\ell)$ to provide contextual information. Let $\mathcal{M}_\ell\subseteq \mathcal{N}(\ell)$ denote the set of geoentities sampled for masking. We replace each target's semantic representation with a learned mask vector:
\begin{equation}
\mathbf{e}_{i,\ell}^{\mathrm{sem}}
=
\mathbf{e}^{\mathrm{mask}} \ \text{if } v_i\in\mathcal{M}_\ell,
\quad
\mathbf{e}_{i,\ell}^{\mathrm{sem}}
=
\mathbf{e}_i^{\mathrm{sem}} \ \text{otherwise},
\label{eq:mask_sem_emb}
\end{equation}
\modelname projects the masked semantic representation and recomputes the adaptive fusion in~\Cref{eq:fused_emb}, yielding the semantic and fused embeddings $\mathbf{z}_{i,\ell}^{\mathrm{sem}}$ and $\mathbf{z}_{i,\ell}^{\mathrm{fuse}}$. Prior approaches to contextual geoentity representation learning represent geographic position using point coordinates, often by adding location embeddings to token representations, primarily capturing point-based spatial proximity~\citep{li-etal-2022-spabert,li-etal-2023-geolm}. In contrast, \modelname represents heterogeneous vector geoentities by their complete geometry (e.g., location, shape, and spatial extent).

To reconstruct masked semantics, spatial-context-aware attention learns how these spatial relations among neighboring entities modulate the strength of contextualization. For clarity, we describe one attention head of the first layer and omit head and layer indices; subsequent layers operate on the preceding layer's outputs. Queries and keys are computed from the fused geoentity embeddings:$
\mathbf{q}_{i,\ell}
=
\mathbf{W}_{Q}\mathbf{z}_{i,\ell}^{\mathrm{fuse}},
\mathbf{k}_{j,\ell}
=
\mathbf{W}_{K}\mathbf{z}_{j,\ell}^{\mathrm{fuse}}.
$
Standard dot-product attention scores a pair $(v_i, v_j)$ using only their query and key representations, which do not explicitly capture the pairwise spatial relation between them. To let the attention score reflect this relation, we introduce a learned \textbf{topology-modulated distance bias} $b_{ij}$:
\begin{equation}
\vspace{-.2em}
a_{ij,\ell}
=
\underbrace{
\frac{
\mathbf{q}_{i,\ell}^{\top}\mathbf{k}_{j,\ell}
}{
\sqrt{d_{\mathrm{head}}}
}
}_{\text{from fused geoentity representations}}
+
\underbrace{
b_{ij}
}_{\text{from their spatial relation}}
,
\qquad
A_{ij,\ell}
=
\frac{
\exp\!\left(a_{ij,\ell}\right)
}{
\displaystyle
\sum_{v_k\in\mathcal{N}(\ell)}
\exp\!\left(a_{ik,\ell}\right)
},
\label{eq:attention}
\vspace{-.2em}
\end{equation}
where $d_{\mathrm{head}}$ is the dimension of one attention head. A larger $b_{ij}$ increases the contribution of $v_j$ relative to other entities in the context, and a smaller value decreases its contribution. Adding pairwise structure to attention scores as a bias has proven effective in other domains, such as shortest-path distances between nodes in the graph domain~\citep{ying2021transformers} and 3D Euclidean distances between atoms conditioned on atom-pair types in the molecular domain~\citep{zhou2023unimol}. For heterogeneous geoentities, we design a bias that jointly captures their distance and topological relation.

The proposed bias is derived directly from geoentity geometry. For each pair $(v_i,v_j)$, $\bar{d}_{ij}$ denotes the normalized minimum distance between their complete geometry and
$t_{ij}\in\mathcal{T}
=\{\textsc{disjoint},\textsc{contact},\textsc{containment}\}$
denotes their symmetric topological relation (see \appxref{appdx:topo}). These quantities are fixed geometric properties of the pair and require neither manual annotations nor downstream labels. We model topology as a symmetric predicate, since ``a restaurant within a mall'' and ``a mall contains a restaurant'' indicate the same shared functional context, and collapsing such inverse relations yields a simpler and denser relational representation across heterogeneous geoentities. Following~\Cref{sec:theory_spatial_dependencies}, we instantiate the topology-modulated distance bias so that topology conditions the effect of distance: 
\begin{equation}
\vspace{-.2em}
b_{ij}
=
\underbrace{
\beta_{t_{ij}}
}_{\text{topology-specific offset}}
+
\underbrace{
\mathbf{w}_{t_{ij}}^{\top}\boldsymbol{\phi}(\bar{d}_{ij})
}_{\text{distance effect under topology }t_{ij}},
\label{eq:relation_bias}
\vspace{-.2em}
\end{equation}
where $\beta_{t}$ is a learned offset and $\mathbf{w}_{t}$ is a learned weight vector for each topological relation $t \in \mathcal{T}$, and $\boldsymbol{\phi}(\cdot)$ denotes Gaussian radial basis features (see~\appxref{appdx:rbf}). The distance term models distance decay, and conditioning $\beta_t$ and $\mathbf{w}_t$ on $t_{ij}$ yields topological modulation, so the same pairwise distance affects attention differently for \textsc{disjoint}, \textsc{contact}, and \textsc{containment} pairs. 
Both \textsc{contact} and \textsc{containment} pairs have $\bar{d}_{ij}=0$, but their learned distance effects under topology $t_{ij}$ can still differ since the weights $\mathbf{w}_{t_{ij}}$ are topology-specific. Thus, both the learned topology-specific offset $\beta_{t_{ij}}$ and the learned distance effect $\mathbf{w}_{t_{ij}}^\top\phi(0)$ can distinguish these relations. For \textsc{disjoint} pairs, the distance effect under topology $t_{ij}$ can additionally vary with $\bar{d}_{ij}$.
These parameters are optimized jointly with the encoder rather than following a predetermined distance-decay function. They are learned separately for each attention head and shared across layers, so different heads can capture distinct topology–distance effects. Within each head, the same spatially guided attention weights aggregate the semantic and fused geoentity embeddings through separate value projections: $\tilde{\mathbf{h}}_{i,\ell}^{s}=\sum_{v_j\in\mathcal{N}(\ell)}A_{ij,\ell}\mathbf{W}_{V,s}\mathbf{z}_{j,\ell}^{s}$ for $s\in\{\mathrm{sem},\mathrm{fuse}\}$.
Sharing the attention map ties both embeddings to the same neighbor selection, while the separate value projections $\mathbf{W}_{V,\mathrm{sem}}$ and $\mathbf{W}_{V,\mathrm{fuse}}$ let each specialize. After multi-head aggregation and the remaining Transformer layers, each with separate semantic and fused feed-forward networks (\appxref{appdx:implementation}), we denote the final outputs by the semantic contextual embedding $\mathbf{h}_{i,\ell}^{\mathrm{sem}}$ and the fused contextual embedding $\mathbf{h}_{i,\ell}^{\mathrm{fuse}}$. The semantic contextual embedding carries spatially contextualized semantics used to reconstruct masked attributes, described below, while the fused contextual embedding additionally carries geometry, from which spatial relations are modeled (\Cref{sec:geom_loss}).


\maskedmodeling reconstructs the original semantic embedding of each masked target as $\hat{\mathbf{e}}_{i,\ell}^{\mathrm{sem}} = f_{\mathrm{rec}}(\mathbf{h}_{i,\ell}^{\mathrm{sem}})$ for $v_i\in\mathcal{M}_\ell$. Reconstruction uses a contrastive objective in which each predicted embedding $\hat{\mathbf{e}}_{i,\ell}^{\mathrm{sem}}$ is pulled toward the target's original embedding $\mathbf{e}_i^{\mathrm{sem}}$ and pushed away from the original semantic embeddings of other geoentities in the same local context:
\begin{equation}
\vspace{-.2em}
\mathcal{L}_{\text{\maskedmodeling}}
=
-\frac{1}{\sum_{\ell}|\mathcal{M}_\ell|}
\sum_{\ell}
\sum_{i\in\mathcal{M}_\ell}
\log
\frac{
\exp\!\left(
\operatorname{sim}
(\hat{\mathbf{e}}_{i,\ell}^{\mathrm{sem}},\mathbf{e}_i^{\mathrm{sem}})
/\tau_{\text{\maskedmodeling}}
\right)
}{
\displaystyle
\sum_{j\in\mathcal{P}_\ell(i)}
\exp\!\left(
\operatorname{sim}
(\hat{\mathbf{e}}_{i,\ell}^{\mathrm{sem}},\mathbf{e}_j^{\mathrm{sem}})
/\tau_{\text{\maskedmodeling}}
\right)
},
\label{eq:mgsm_loss}
\vspace{-.2em}
\end{equation}
where $\mathcal{P}_\ell(i)$ is the candidate set for $v_i$. It contains the target as the positive and the other geoentities in $\mathcal{N}(\ell)$ as negatives, with any geoentity whose semantic embedding is identical to the target's removed to avoid false negatives. $\operatorname{sim}(\cdot,\cdot)$ denotes cosine similarity and $\tau_{\text{\maskedmodeling}}>0$ is the temperature.
\vspace{-1em}
\subsection{Geometry-aware Spatial Relation Modeling}\label{sec:geom_loss}
\vspace{-.5em}
The spatial relation bias uses pairwise distance and topology to shape attention, yet these relations are not explicitly encoded in the representations. We introduce \textbf{Geometry-Aware Spatial Relation Modeling (\georelmodeling)}, a self-supervised objective that serves two purposes. First, \georelmodeling guides the fused contextual embeddings to encode pairwise distance and topology. Second, the gradient propagates back through the encoder layers and fused geoentity embeddings to the geometry encoder, grounding the geometry embeddings in spatial relations. The idea is that semantic context allows NARA to model these relations jointly with functional context rather than from geometry alone, such as containment versus adjacency, as in a shop inside a commercial building versus a parking lot beside it. \georelmodeling therefore predicts these relations from the fused contextual embeddings, so that each embedding jointly encodes a geoentity's semantic context and its spatial relations to neighbors.


\modelname samples two sets of unmasked geoentity pairs within each local context, for topology classification and distance regression, respectively. Topology pairs are stratified jointly by topology class and geometry-type pair, and distance pairs by distance interval and geometry-type pair, to mitigate class imbalance, as pairs are dominated by \textsc{disjoint} and common geometry types (see~\appxref{appdx:implementation}). For each sampled pair $(v_i,v_j)$ in context $\ell$, we construct the pair representation $\mathbf{p}_{ij,\ell}=\left[\mathbf{h}_{i,\ell}^{\mathrm{fuse}}\odot\mathbf{h}_{j,\ell}^{\mathrm{fuse}};\left|\mathbf{h}_{i,\ell}^{\mathrm{fuse}}-\mathbf{h}_{j,\ell}^{\mathrm{fuse}}\right|\right]$,
which is order-invariant, matching the symmetry of $t_{ij}$ and $\bar{d}_{ij}$. For each sampled topology pair and distance pair, \georelmodeling minimizes, respectively,
\begin{equation}
\vspace{-.2em}
\operatorname{CE}\!\left(f_{\mathrm{topo}}(\mathbf{p}_{ij,\ell}),\,t_{ij}\right)
\qquad\text{and}\qquad
\left(f_{\mathrm{dist}}(\mathbf{p}_{ij,\ell})-\bar{d}_{ij}\right)^2.
\label{eq:spatial_relation_losses}
\vspace{-.2em}
\end{equation}
$\mathcal{L}_{\mathrm{topo}}$ and $\mathcal{L}_{\mathrm{dist}}$ average these terms over the sampled topology and distance pairs, and are combined as $\mathcal{L}_{\text{\georelmodeling}} = \alpha_{\mathrm{topo}}\mathcal{L}_{\mathrm{topo}} + \alpha_{\mathrm{dist}}\mathcal{L}_{\mathrm{dist}}$. \maskedmodeling trains the semantic contextual embeddings through reconstruction, and \georelmodeling grounds the fused contextual embeddings and the geometry encoder in spatial relations. Because both objectives share one attention map, \modelname learns which neighbors inform each target while keeping their distance and topology recoverable from the fused contextual embeddings for downstream tasks.
\vspace{-1em}
\subsection{Relation-Conditioned Regularization}\label{sec:relation_conditioned}
\vspace{-.5em}

We introduce two complementary objectives that promote relation-conditioned contextualization (\Cref{eq:topo_modulated_kernel}). The first, \textbf{Anchor-Conditioned Contrastive Learning (ACC)}, focuses on \emph{anchor-conditioned relations} by encouraging \modelname to pull sibling entities sharing the same relation to an anchor closer in the contextual semantic embedding space, while pushing non-siblings apart, with attraction modulated by geographic distance. The second, \textbf{Relational Semivariogram Regularization (RSR)}, inspired by Kriging for spatial prediction~\citep{lin2020deep}, takes a broader \emph{neighborhood-level} view by comparing anchor-conditioned sibling pairs with nearby, geometry-type-matched non-sibling pairs within the same geographic distance bins. It encourages the average similarity of sibling pairs to exceed the non-sibling baseline by a margin in each valid distance bin, reinforcing relation-conditioned similarity beyond metric proximity alone.

For an anchor $a$ and a member $v_i$, $\mathrm{rel}(v_i, a)\in\mathcal{R}$ denotes their relation, where $\mathcal{R}$ consists of polygons along a polyline, points along a polyline, and points inside a polygon (e.g., buildings or POIs near a road, and POIs inside a building; see~\appxref{appdx:anchor}). This anchor relation is directed, since it specifies which geoentity plays the role of the shared anchor. Following~\Cref{eq:relational_autocorr}, the members that hold the same relation to the same anchor form a \emph{sibling group}, denoted as $\mathcal{S}_{a,r}=\left\{v_i\in\mathcal{N}(\ell)\setminus\mathcal{M}_\ell\;\middle|\;\mathrm{rel}(v_i,a)=r\right\}.$
The siblings $\mathcal{S}_i$ of $v_i$ are the other members of every sibling group that contains $v_i$. A geoentity can therefore participate in several relational contexts, and, by~\Cref{eq:relational_autocorr}, siblings are expected to be more similar than other geoentities at the same distance, motivating the two objectives below.


\noindent \textbf{Anchor-Conditioned Contrastive Learning.} ACC enforces relation-conditioned similarity by pulling sibling pairs together and pushing non-siblings apart, with attraction modulated by distance. Let $c_{ij}$ denote the cosine similarity between $\mathbf{h}_{i,\ell}^{\mathrm{sem}}$ and $\mathbf{h}_{j,\ell}^{\mathrm{sem}}$, and let $\mathcal{E}_i$ denote the unmasked geoentities in $\mathcal{N}(\ell)$ sharing the geometry type of $v_i$, excluding $v_i$. The positives are the $v_i$'s siblings, and the negatives are the remaining geoentities in $\mathcal{E}_i$. Because nearby siblings are expected to be more similar than distant ones (\Cref{eq:variogram_def}), each sibling receives a weight $\omega_{ij}\propto\exp(-d_{ij}/\lambda)$ with scale $\lambda>0$, normalized to sum to one over $\mathcal{S}_i$. For each geoentity $v_i$ with at least one sibling, ACC minimizes
\begin{equation}
\vspace{-.2em}
-\sum_{v_j\in\mathcal{S}_i}
\omega_{ij}
\log
\frac{\exp\!\left(c_{ij}/\tau_{\text{ACC}}\right)}
{\sum_{v_k\in\mathcal{E}_i}\exp\!\left(c_{ik}/\tau_{\text{ACC}}\right)},
\label{eq:acc_loss}
\vspace{-.2em}
\end{equation}
where $d_{ij}$ denotes complete-geometry distance and $\tau_{\text{ACC}}>0$ is the temperature. $\mathcal{L}_{\text{ACC}}$ averages
this term over such geoentities within each context and then over contexts. Restricting negatives to the same geometry type keeps the contrast from being dominated by differences between points and polygons. This formulation allows entities to belong to multiple relational contexts, with distance-aware weighting ensuring proximate siblings exhibit stronger similarity than distant siblings or non-siblings at the same distance.

\vspace{-.5em}
\noindent \textbf{Relational Semivariogram Regularization.} 
Building on~\Cref{eq:variogram_def} and Kriging for spatial prediction~\citep{lin2020deep}, RSR employs semivariogram-based regularization over $\ell_2$-normalized contextual semantic embeddings $\mathbf{h}_{i,\ell}^{\mathrm{sem}}$. RSR compares each sibling group's trend with that of a relation-agnostic baseline at a matched distance, requiring siblings to be more similar than the baseline at every distance, as expected by~\Cref{eq:relational_autocorr}. \modelname uses cosine distance as a proxy for squared Euclidean distance (see~\appxref{appdx:semivariogram}), the empirical semivariance over a pair set $\Pi$ is:
$\hat{\gamma}(\Pi)=\frac{1}{|\Pi|}\sum_{(v_i,v_j)\in\Pi}\left(1-c_{ij}\right).$
\modelname computes two semivariance estimates conditioned on a bin $b$ of a geographic distance $d_{ij}$. For a sibling group $\mathcal{S}_{a,r}$ and bin $b$, the relational semivariance $\gamma_{\mathrm{rel}}^{(a,r)}(b)$ is $\hat{\gamma}$ over the sibling pairs of $\mathcal{S}_{a,r}$ in bin $b$. The neighborhood semivariance $\gamma_{\mathrm{glob}}^{(\eta_{a,r})}(b)$ serves as a relation-agnostic baseline, computed as $\hat{\gamma}$ over nearby non-sibling pairs in the same context and bin whose geometry type $\eta_{a,r}$ matches the group's members, and captures distance decay without relational structure. A gap between the two at the same distance therefore reflects the shared anchor relation. For each sibling group $\mathcal{S}_{a,r}$, RSR penalizes
\begin{equation}
\sum_{b}
\pi_b^{(a,r)}\,
\max\!\left(0,\;
\gamma_{\mathrm{rel}}^{(a,r)}(b)-\gamma_{\mathrm{glob}}^{(\eta_{a,r})}(b)+\delta\right),
\label{eq:rsr_loss}
\end{equation}
where the sum runs over valid distance bins containing sufficient sibling and non-sibling pairs, and \(\delta>0\) is a margin. The weight \(\pi_b^{(a,r)}\) reflects the relative number of sibling pairs assigned to bin \(b\). \(\mathcal L_{\mathrm{RSR}}\) aggregates these group-level penalties across sibling groups and spatial contexts.  

Overall, \modelname pretraining loss combines all four objectives:
\begin{equation}
\mathcal{L} = \mathcal{L}_{\text{\maskedmodeling}} + \mathcal{L}_{\text{\georelmodeling}} + \alpha_{\text{ACC}}\,\mathcal{L}_{\text{ACC}} + \alpha_{\text{RSR}}\,\mathcal{L}_{\text{RSR}},
\label{eq:joint_loss}
\end{equation}
where $\alpha_{\text{ACC}}$, $\alpha_{\text{RSR}} \geq 0$ weight the objectives, and $\mathcal{L}_{\text{\georelmodeling}}$ carries its own internal weights $\alpha_{\mathrm{topo}},\alpha_{\mathrm{dist}}$.

\vspace{-1em}
\section{Experiments}\label{sec:exp}\vspace{-1em}
We evaluate \modelname on three standard tasks spanning all geometry types: traffic speed estimation on polylines (\Cref{sec:exp_line}), building function classification on polygons (\Cref{sec:exp_building}), and next POI prediction on points (\Cref{sec:exp_point}). In all three tasks, the state-of-the-art method relies on task-specific feature engineering or geometry-type-specific architectures. In contrast, \modelname uses a single task-agnostic architecture with shared hyperparameters, pretrained separately on OSM data from New York City (NYC) and Singapore (SG). No downstream labels are used during \modelname pretraining. For downstream evaluation, the pretrained encoder is frozen, and each target geoentity is contextualized with the point, polyline, and polygon geoentities intersecting its 100-meter buffer. The target's own semantic representation is replaced with the mask vector, as in \maskedmodeling, so \modelname infers the target from its geometry and spatial context rather than from its own attributes. Its contextual representation $[\mathbf{h}_i^{\mathrm{sem}};\mathbf{h}_i^{\mathrm{fuse}}]$ is passed to a lightweight task-specific head. Experimental details are in Appdx.~\ref{appdx:implementation}--\ref{appdx:baseline}.
\vspace{-1em}
\subsection{Task A: Traffic Speed Estimation}\label{sec:exp_line}\vspace{-.5em}
This task predicts the average observed traffic speed on road segments, each represented as a polyline geoentity, using Uber Movement data mapped to OSM road segments. Traffic speed depends on the spatial context, surrounding land use and topological relations with nearby geoentities. \modelname uses the same frozen, task-agnostic encoder as in the other two tasks, without road-specific architectures or pretraining objectives. It encodes each segment of a road and average-pools the segment embeddings into a road-level representation, handling arbitrary road lengths and segment granularities without architectural changes. In contrast, the state-of-the-art CityFM~\citep{balsebre2024city} is tailored to this task. It fine-tunes BERT~\citep{devlin-etal-2019-bert} on OSM tags of point and polygon entities, represents polylines with engineered features (e.g., segment length, vertex count, and one-hot road type), and uses the traffic speeds of neighboring roads from the training set as input.
We also compare against road-network graph methods (Node2Vec~\citep{grover2016node2vec}, GCWC~\citep{hu2019stochastic}, RFN~\citep{jepsen2020relational}), geometry encoders (Poly2Vec~\citep{siampou2025polyvec}, Geo2Vec~\citep{chu2026geo2vec}), and spatial representation methods (IRN2Vec~\citep{wang2019learning}, GeoVectors~\citep{tempelmeier2021geovectors}). For a fair comparison with CityFM, \modelname's downstream head receives the frozen \modelname contextual representation concatenated with engineered road features and training-set neighboring speeds, both from CityFM~\cite{balsebre2024city}.

\textbf{Results.} \modelname achieves the best results on every metric, reducing RMSE by 6.62\%, MAE by 8.44\%, and MAPE by 8.25\% relative to the strongest baseline, CityFM, and raising $R^2$ from 0.77 to 0.80 (\Cref{tab:traffic_speed}). All improvements are significant under Holm-corrected two-sided paired $t$-tests ($p<0.01$). \modelname improves on it solely by replacing CityFM's polyline-specific pipeline with its frozen, task-agnostic representation. The baseline categories show how each source of spatial context contributes: topology-only models (MAE 4.83--5.31) capture connectivity but lack semantics and geometry; geometry encoders (4.20--4.24) add shape; and semantic methods (3.78--3.92) add road attributes. CityFM (MAE 3.20, MAPE 19.27\%) shows high performance among baselines with a polyline-specific contextualization pipeline. Across road types (\appxref{appdx:polyline}), \modelname lowers MAE relative to CityFM in 13 of 16 types, by 5.4--9.9\% on the four most frequent (residential, primary, secondary, and tertiary; 83\% of test roads) and by 11.9\% and 15.9\% on motorways and trunk roads, respectively. This suggests that modeling heterogeneous spatial context improves robustness across common urban road environments.

\begin{table}[h]
\vspace{-.5em}
\centering
\caption{
Traffic speed prediction: results averaged over 10 independent runs (SD in parentheses). Gains are computed relative to the second-best method. \textbf{Bold} = best, \underline{underline} = second-best.
}
\label{tab:traffic_speed} \scalebox{0.76}{
\begin{tabular}{llcccc}
\toprule
Category & Model & RMSE $\downarrow$ & MAE $\downarrow$ & R$^2$ $\uparrow$ & MAPE $\downarrow$ \\
\midrule
\textit{Topology (Road Network)} 
& Node2Vec
& 6.82 ($\pm$ 0.12) & 5.31 ($\pm$ 0.05) & 0.38 ($\pm$ 0.04) & 32.22\% ($\pm$ 0.6\%) \\
& GCWC
& 6.74 ($\pm$ 0.04) & 5.20 ($\pm$ 0.04) & 0.41 ($\pm$ 0.04) & 32.75\% ($\pm$ 1.2\%) \\
& RFN
& 6.45 ($\pm$ 0.09) & 4.83 ($\pm$ 0.02) & 0.46 ($\pm$ 0.02) & 30.10\% ($\pm$ 0.0\%) \\

\midrule
\textit{Geometry} 
& Poly2Vec
& 5.66 ($\pm$ 0.01) & 4.24 ($\pm$ 0.01) & 0.56 ($\pm$ 0.00) & 25.75\% ($\pm$ 0.1\%) \\
& Geo2Vec
& 5.64 ($\pm$ 0.01) & 4.20 ($\pm$ 0.01) & 0.56 ($\pm$ 0.00) & 25.42\% ($\pm$ 0.2\%) \\

\midrule
\textit{Semantics + Metric} 
& IRN2Vec
& 5.02 ($\pm$ 0.04) & 3.78 ($\pm$ 0.02) & 0.66 ($\pm$ 0.01) & 24.30\% ($\pm$ 0.0\%) \\
& GeoVectors
& 5.21 ($\pm$ 0.07) & 3.92 ($\pm$ 0.06) & 0.64 ($\pm$ 0.06) & 24.03\% ($\pm$ 0.3\%) \\

\midrule
\textit{Semantics + Geometry} 
& CityFM
& \underline{4.08} ($\pm$ 0.01) & \underline{3.20} ($\pm$ 0.01) & \underline{0.77} ($\pm$ 0.02) & \underline{19.27\%} ($\pm$ 0.8\%) \\
\rowcolor{gray!10} 

& \textbf{\modelname} 
& \makecell{\textbf{3.81} ($\pm$ 0.01) \\ {+6.62\%}} 
& \makecell{\textbf{2.93} ($\pm$ 0.01) \\ {+8.44\%}} 
& \makecell{\textbf{0.80} ($\pm$ 0.00) \\ {+3.90\%}} 
& \makecell{\textbf{17.68\%} ($\pm$ 0.1\%) \\ {+8.25\%}} \\

\bottomrule
\end{tabular}}
\vspace{-1em}
\end{table}

\vspace{-.5em}
\subsection{Task B: Building Function Classification}
\label{sec:exp_building}\vspace{-.5em}
This task predicts land-use categories for untagged OSM building polygons using the standard 8-class Singapore government land-use taxonomy, following the CityFM protocol~\citep{balsebre2024city}. A building's function depends not only on its footprint but on its spatial relations with surrounding POIs, roads, and neighboring buildings. \modelname uses the frozen, task-agnostic encoder and classifies each building from its contextual representation alone. As for every downstream target, the building's semantic representation is masked, so the prediction relies only on its footprint and surrounding context. In contrast, the state-of-the-art CityFM is tailored to this task, adding a building-type-aligned pretraining objective that contrasts rasterized building footprints with OSM building-type labels and relying on engineered features such as precomputed building area. We also compare against non-spatial semantic models (BERT~\citep{devlin-etal-2019-bert}), geometry encoders (Poly2Vec~\citep{siampou2025polyvec}, Geo2Vec~\citep{chu2026geo2vec}), and spatial methods with neighborhood context (SpaBERT~\citep{li-etal-2022-spabert}, GeoVectors~\citep{tempelmeier2021geovectors}).

\textbf{Results.} \modelname achieves the best results on every metric, improving Macro-F1 by 7.15\% (69.34 $\rightarrow$ 74.30), Weighted-F1 by 1.99\%, and accuracy by 1.44\% relative to the strongest baseline, CityFM (\Cref{tab:building}). All improvements are significant under Holm-corrected two-sided paired $t$-tests ($p<0.001$). Unlike CityFM, \modelname uses neither a building-specific pretraining objective nor engineered building features, relying only on its task-agnostic contextual representation. Since the target's semantics are masked, these gains come from its geometry and spatial context, which \maskedmodeling trains the encoder to exploit. The baseline categories show how each source of spatial context contributes: BERT (Macro-F1 37.56) shows semantic attributes alone are insufficient; geometry methods (45.12--58.35) show footprint shape is highly informative; and point-based spatial methods (GeoVectors, SpaBERT) underperform them, as point-level proximity fails to capture building extent and shape. Per class (\appxref{appdx:polygon}), \modelname's Macro-F1 gain comes mainly from underrepresented classes whose function cannot be inferred from geometry alone: Civic \& Community Institutions (+27.49 F1), Commercial \& Residential (+20.61), and Educational (+16.74).

\begin{table}[h]
\centering
\vspace{-1em}
\caption{
Building function classification: results averaged over 10 folds (SD in parentheses). Gains are computed relative to the second-best method. \textbf{Bold} = best, \underline{underline} = second-best.
}
\label{tab:building} 
\scalebox{0.77}{
\begin{tabular}{llccc}
\toprule
Category & Model & Macro-F1 $\uparrow$ & Weighted-F1 $\uparrow$ & Accuracy $\uparrow$ \\
\midrule
\textit{Non-Spatial Semantics} 
& BERT (fine-tuned) 
& 37.56 ($\pm$ 0.3) & 44.19 ($\pm$ 0.7) & 51.77 ($\pm$ 0.7) \\

\midrule
\textit{Geometry} 
& Poly2Vec 
& 45.12 ($\pm$ 2.9) & 84.25 ($\pm$ 1.0)& 86.65 ($\pm$ 0.8) \\
& Geo2Vec
& 58.35 ($\pm$ 0.7) & 86.90 ($\pm$ 0.3) & 88.05 ($\pm$ 0.3) \\

\midrule
\textit{Semantics + Metric} 
& GeoVectors
& 47.24 ($\pm$ 1.4) & 64.18 ($\pm$ 2.1) & 69.49 ($\pm$ 1.5) \\
& SpaBERT
& 45.06 ($\pm$ 1.1) & 78.47 ($\pm$ 2.4) & 75.95 ($\pm$ 2.7) \\

\midrule
\textit{Semantics + Geometry} 
& CityFM
& \underline{69.34} ($\pm$ 2.1) & \underline{90.95} ($\pm$ 0.4) & \underline{91.81} ($\pm$ 0.3) \\
\rowcolor{gray!10} 
& \textbf{\modelname} 
& \makecell{\textbf{74.30} ($\pm$ 1.2) \\ {+ 7.15\%}} 
& \makecell{\textbf{92.76} ($\pm$ 0.3) \\ {+ 1.99\%}} 
& \makecell{\textbf{93.13} ($\pm$ 0.3) \\ {+ 1.44\%}} \\

\bottomrule
\end{tabular}}
\vspace{-1em}
\end{table}

\vspace{-.5em}
\subsection{Task C: Next Point-of-Interest Prediction}\label{sec:exp_point}
\vspace{-.5em}
This task predicts the next POI a user will visit from a historical check-in sequence, using Foursquare data from NYC and SG~\citep{yang2014modeling} and following the POI-Enhancer protocol~\citep{cheng2025poi}. POI visit dynamics depend on surrounding context: a cafe in a commercial district attracts different visits from one in a residential area despite sharing a category. \modelname's frozen, task-agnostic encoder jointly encodes each Foursquare POI with the OSM geoentities within its 100-meter buffer, without check-in data or other task-specific inputs. Because \modelname is pretrained only on OSM, this setting also tests transfer to a vector data source unseen during pretraining (i.e., Foursquare). In contrast, the state-of-the-art POI-Enhancer~\cite{cheng2025poi} is tailored to this task: it enriches POI embeddings by prompting Llama-2-7B~\citep{touvron2023llama} with task-specific inputs, including check-in statistics, OSM addresses (e.g., street name, postal code), and nearby-POI categories. Following POI-Enhancer, we fuse the resulting embeddings into two trajectory models, TALE~\citep{wan2021tale} and CTLE~\citep{lin2021pre}, via cross-attention, and compare each trajectory model alone, augmented with POI-Enhancer or \modelname.

\textbf{Results.} \modelname achieves the best Hit@1 and Hit@5 in all eight city--model--metric combinations, improving over POI-Enhancer by 3.25--17.38\% and over the base trajectory models by up to 43.98\% (NY, TALE, Hit@1) (\Cref{tab:next_poi}). The contrast is sharpest in Singapore, where POI-Enhancer yields little or negative benefit ($-2.60$\% to $+1.54$\%), whereas \modelname improves every setting ($+2.57$\% to $+13.89$\%). \modelname thus outperforms an LLM-based method that uses task-specific inputs such as check-in statistics, while relying only on pretrained OSM spatial context. Per venue type in NYC (\appxref{appdx:point_type}), \modelname improves Hit@1 in 13 of the 15 most frequent categories (e.g., Subway, Bar) and ties in the remaining two, and improves Hit@5 in all 15. 
\begin{table}[h]
\vspace{-1em}
\centering
\caption{
Next POI prediction: results averaged over 3 independent runs. Parentheses indicate relative improvement over the base model (TALE/CTLE); $\Delta$ denotes gain over POI-Enhancer. \textbf{Bold} = best.
}
\label{tab:next_poi}
\scalebox{0.74}{
\begin{tabular}{lllll}
\toprule
\multirow{2}{*}{Model}
& \multicolumn{2}{c}{NY}
& \multicolumn{2}{c}{SG} \\
\cmidrule(lr){2-3} \cmidrule(lr){4-5}
& \multicolumn{1}{c}{Hit@1 $\uparrow$}
& \multicolumn{1}{c}{Hit@5 $\uparrow$}
& \multicolumn{1}{c}{Hit@1 $\uparrow$}
& \multicolumn{1}{c}{Hit@5 $\uparrow$} \\
\midrule

TALE
& 5.680 $\pm$ 0.184
& 12.170 $\pm$ 0.350
& 9.669 $\pm$ 0.073
& 21.801 $\pm$ 0.073 \\

+ POI-Enhancer
& 7.728 $\pm$ 0.265
  {\scriptsize\makecell{(+36.06\%)}}
& 15.928 $\pm$ 0.051
  {\scriptsize\makecell{(+30.87\%)}}
& 9.690 $\pm$ 0.069
  {\scriptsize\makecell{(+0.22\%)}}
& 21.235 $\pm$ 0.246
  {\scriptsize\makecell{(-2.60\%)}} \\

\rowcolor{gray!10}
+ \textbf{\modelname}
& \textbf{8.178} $\pm$ 0.087
  {\scriptsize\makecell{(+43.98\%)\\($\Delta$5.82\%)}}
& \textbf{17.121} $\pm$ 0.124
  {\scriptsize\makecell{(+40.67\%)\\($\Delta$7.49\%)}}
& \textbf{11.012} $\pm$ 0.275
  {\scriptsize\makecell{(+13.89\%)\\($\Delta$13.64\%)}}
& \textbf{24.460} $\pm$ 0.283
  {\scriptsize\makecell{(+12.20\%)\\($\Delta$15.19\%)}} \\

\midrule

CTLE
& 6.782 $\pm$ 0.300
& 13.735 $\pm$ 0.678
& 10.607 $\pm$ 0.064
& 22.922 $\pm$ 0.194 \\

+ POI-Enhancer
& 7.076 $\pm$ 0.031
  {\scriptsize\makecell{(+4.34\%)}}
& 14.324 $\pm$ 0.154
  {\scriptsize\makecell{(+4.29\%)}}
& 10.529 $\pm$ 0.066
  {\scriptsize\makecell{(-0.74\%)}}
& 23.275 $\pm$ 0.055
  {\scriptsize\makecell{(+1.54\%)}} \\

\rowcolor{gray!10}
+ \textbf{\modelname}
& \textbf{8.072} $\pm$ 0.085
  {\scriptsize\makecell{(+19.02\%)\\($\Delta$14.06\%)}}
& \textbf{16.814} $\pm$ 0.176
  {\scriptsize\makecell{(+22.43\%)\\($\Delta$17.38\%)}}
& \textbf{10.880} $\pm$ 0.103
  {\scriptsize\makecell{(+2.57\%)\\($\Delta$3.33\%)}}
& \textbf{24.033} $\pm$ 0.128
  {\scriptsize\makecell{(+4.85\%)\\($\Delta$3.25\%)}} \\

\bottomrule
\end{tabular}
}
\vspace{-1em}
\end{table}

\vspace{-.5em}
\subsection{Ablation Studies}\label{sec:ablation}
\vspace{-.5em}
We ablate each component on building function classification (\Cref{tab:unified_ablation}), since inferring a building's function requires both its geometry and spatial relations to surrounding geoentities, and thus exercises every part of \modelname. The topology-modulated distance bias improves Macro-F1 from 68.46 to 69.94, and GSRM yields the largest additional gain to 72.93; ACC and RSR further improve performance to 73.98 and 74.30. These results show that \modelname's gains come from contextualizing each geoentity with the semantics and geometry of its neighbors through their spatial relations, rather than from richer geometry alone. 


\begin{table}[h]
\vspace{-.5em}
\centering
\caption{
\modelname ablations on building function classification. Each row adds one component to the row above. T/D bias denotes the topology-modulated distance bias.
}
\label{tab:unified_ablation}
\scalebox{0.82}{
\begin{tabular}{@{}lccc@{}}
\toprule
Configuration & Macro-F1 $\uparrow$ & Weighted-F1 $\uparrow$ & Accuracy $\uparrow$ \\
\midrule
\maskedmodeling                     & 68.463 $\pm$ 2.135 & 90.617 $\pm$ 0.608 & 91.218 $\pm$ 0.456 \\
\quad + T/D bias                    & 69.943 $\pm$ 1.482 & 91.677 $\pm$ 0.360 & 92.189 $\pm$ 0.315 \\
\quad + \georelmodeling             & 72.934 $\pm$ 1.243 & 92.371 $\pm$ 0.421 & 92.793 $\pm$ 0.340 \\
\quad + ACC                         & \underline{73.982 $\pm$ 1.625} & \underline{92.757 $\pm$ 0.335} & \textbf{93.142 $\pm$ 0.280} \\
\quad + RSR (\textbf{Full \modelname}) & \textbf{74.298 $\pm$ 1.153} & \textbf{92.764 $\pm$ 0.306} & \underline{93.134 $\pm$ 0.296} \\
\bottomrule
\end{tabular}}
\vspace{-1em}
\end{table}
\vspace{-.5em}
\subsection{Robustness Across Prediction Conditions}\label{sec:robustness}
\vspace{-.5em}
Performance on next POI prediction varies with spatial context and user history (\Cref{tab:poi_slices} in \appxref{appdx:point_slice}). To assess whether \modelname's gains depend on particular conditions, we partition test events by neighborhood density, sibling share, target popularity, history length (number of prior check-ins), and whether the target POI is new or revisited. \modelname outperforms POI-Enhancer~\citep{cheng2025poi} in nearly all slices across both cities and trajectory models, with statistically significant gains in most slices. Gains increase with neighborhood density: in Singapore, the Hit@1 gain rises from sparse to dense neighborhoods, reaching $+27.4\%$ for TALE. The only slices without improvement occur under SG-CTLE, and none of these differences is statistically significant.

\vspace{-1em}
\section{Conclusion}
\vspace{-1em}
We present \modelname, a self-supervised framework for contextual representation learning over heterogeneous vector geoentities. \modelname contextualizes geoentities through spatial-context-aware attention with a topology-modulated distance bias, and learns semantic--geometric representations through masked geoentity semantic modeling, geometry-aware spatial relation modeling, and relation-conditioned regularization grounded in spatial autocorrelation. With a single frozen, task-agnostic encoder, \modelname outperforms task-specialized state-of-the-art methods on traffic speed estimation, building function classification, and next POI prediction, spanning all three geometry types. These results suggest that jointly modeling semantics, complete geometry, and spatial relations provides a strong foundation for task-agnostic contextual geoentity representation learning.

\vspace{-1em}
\subsection*{AI use statement}
\vspace{-1em}
In this work, we used generative AI tools to improve the clarity and readability of the manuscript, including polishing transitions, refining the presentation and framing of tables, and assisting with polishing the implementation pipeline, including the presentation of arguments. Generative AI was not used to formulate scientific claims, develop theoretical arguments, design the overall methodology, or draw research conclusions. AI-assisted code was reviewed, adapted, and tested by the authors for correctness and consistency with the intended implementation. We have reviewed all AI-assisted work, including the manuscript text, presentation, and code, and verified that it accurately reflects the authors' intended content and implementation. We take responsibility for the final content of this work, including text, claims, and artifacts produced with the aid of generative AI.

\vspace{-1em}
\subsection*{Ethics statement}
\vspace{-1em}
All data used in this study are publicly available and were not collected by the authors: OpenStreetMap geoentities, released under the Open Database License; Uber Movement traffic speeds, which are aggregated over trips and road segments and contain no individual trip records; the Singapore land use labels used in the CityFM benchmark~\citep{balsebre2024city}; and the Foursquare check-in datasets of \citet{yang2014modeling}, which are distributed for research use with anonymized user identifiers. No new human-subject data were collected, and we do not attempt to identify individuals or link check-in records to external sources.



\vspace{-1em}
\bibliography{iclr2027_conference}

@INPROCEEDINGS{lin2020deep,
  author={Lin, Yijun and Chiang, Yao-Yi and Franklin, Meredith and Eckel, Sandrah P. and Ambite, José Luis},
  booktitle={2020 IEEE International Conference on Data Mining (ICDM)}, 
  title={Building Autocorrelation-Aware Representations for Fine-Scale Spatiotemporal Prediction}, 
  year={2020},
  volume={},
  number={},
  pages={352-361},
  doi={10.1109/ICDM50108.2020.00044}}

@inproceedings{
siampou2025polyvec,
title={Poly2Vec: Polymorphic Fourier-Based Encoding of Geospatial Objects for Geo{AI} Applications},
author={Maria Despoina Siampou and Jialiang Li and John Krumm and Cyrus Shahabi and Hua Lu},
booktitle={Forty-second International Conference on Machine Learning},
year={2025},
url={https://openreview.net/forum?id=kWyov6XrXs}
}

@inproceedings{
mai2020multi,
title={Multi-Scale Representation Learning  for Spatial Feature Distributions using Grid Cells},
author={Gengchen Mai and Krzysztof Janowicz and Bo Yan and Rui Zhu and Ling Cai and Ni Lao},
booktitle={International Conference on Learning Representations},
year={2020},
url={https://openreview.net/forum?id=rJljdh4KDH}
}

@inproceedings{balsebre2024city,
author = {Balsebre, Pasquale and Huang, Weiming and Cong, Gao and Li, Yi},
title = {City Foundation Models for Learning General Purpose Representations from OpenStreetMap},
year = {2024},
isbn = {9798400704369},
publisher = {Association for Computing Machinery},
address = {New York, NY, USA},
url = {https://doi.org/10.1145/3627673.3679662},
doi = {10.1145/3627673.3679662},
booktitle = {Proceedings of the 33rd ACM International Conference on Information and Knowledge Management},
pages = {87–97},
numpages = {11},
location = {Boise, ID, USA},
series = {CIKM '24}
}

@inproceedings{yang2025hygmap,
author = {Yang, Yifan and Wang, Jingyuan and Yu, Xie and Tang, Yibang},
title = {HygMap: representing all types of map entities via heterogeneous hypergraph},
year = {2025},
isbn = {978-1-956792-06-5},
url = {https://doi.org/10.24963/ijcai.2025/1049},
doi = {10.24963/ijcai.2025/1049},
booktitle = {Proceedings of the Thirty-Fourth International Joint Conference on Artificial Intelligence},
articleno = {1049},
numpages = {9},
location = {Montreal, Canada},
series = {IJCAI '25}
}

@article{cheng2025poi, title={POI-Enhancer: An LLM-based Semantic Enhancement Framework for POI Representation Learning}, volume={39}, url={https://ojs.aaai.org/index.php/AAAI/article/view/33252}, DOI={10.1609/aaai.v39i11.33252}, abstractNote={POI representation learning plays a crucial role in handling tasks related to user mobility data. Recent studies have shown that enriching POI representations with multimodal information can significantly enhance their task performance. Previously, the textual information incorporated into POI representations typically involved only POI categories or check-in content, leading to relatively weak textual features in existing methods. In contrast, large language models (LLMs) trained on extensive text data have been found to possess rich textual knowledge.
However leveraging such knowledge to enhance POI representation learning presents two key challenges: first, how to extract POI-related knowledge from LLMs effectively, and second, how to integrate the extracted information to enhance POI representations.
To address these challenges, we propose POI-Enhancer, a portable framework that leverages LLMs to improve POI representations produced by classic POI learning models. We first design three specialized prompts to extract semantic information from LLMs efficiently. Then, the Dual Feature Alignment module enhances the quality of the extracted information, while the Semantic Feature Fusion module preserves its integrity. The Cross Attention Fusion module then fully adaptively integrates such high-quality information into POI representations and Multi-View Contrastive Learning further injects human-understandable semantic information into these representations. Extensive experiments on three real-world datasets demonstrate the effectiveness of our framework, showing significant improvements across all baseline representations.}, number={11}, journal={Proceedings of the AAAI Conference on Artificial Intelligence}, author={Cheng, Jiawei and Wang, Jingyuan and Zhang, Yichuan and Ji, Jiahao and Zhu, Yuanshao and Zhang, Zhibo and Zhao, Xiangyu}, year={2025}, month={Apr.}, pages={11509-11517} }

@article{lin2021pre, title={Pre-training Context and Time Aware Location Embeddings from Spatial-Temporal Trajectories for User Next Location Prediction}, volume={35}, url={https://ojs.aaai.org/index.php/AAAI/article/view/16548}, DOI={10.1609/aaai.v35i5.16548}, abstractNote={Pre-training location embeddings from spatial-temporal trajectories is a fundamental procedure and very beneficial for user next location prediction. In the real world, a location usually has variable functionalities under different contextual environments. If the exact functions of a location in the trajectory can be reflected in its embedding, the accuracy of user next location prediction should be improved. Yet, existing location embeddings pre-trained on trajectories are mostly based on distributed word representations, which mix a location’s various functionalities into one latent representation vector. To address this problem, we propose a Context and Time aware Location Embedding (CTLE) model, which calculates a location’s representation vector with consideration of its specific contextual neighbors in trajectories. In this way, the multi-functional properties of locations can be properly tackled. Furthermore, in order to incorporate temporal information in trajectories into location embeddings, we propose a subtle temporal encoding module and a novel pre-training objective, which further improve the quality of location embeddings. We evaluate our proposed model on two real-world mobile user trajectory datasets. The experimental results demonstrate that, compared with the existing embedding methods, our CTLE model can pre-train higher quality location embeddings and significantly improve the performance of downstream user location prediction models.}, number={5}, journal={Proceedings of the AAAI Conference on Artificial Intelligence}, author={Lin, Yan and Wan, Huaiyu and Guo, Shengnan and Lin, Youfang}, year={2021}, month={May}, pages={4241-4248} }

@misc{osm,
  author       = {{OpenStreetMap contributors}},
  title        = {Planet dump from OpenStreetMap},
  year         = {2017},
  howpublished = {\url{https://planet.osm.org}},
  note         = {Accessed: 2026-03-26}
}

@misc{overturemaps,
  author       = {{OpenStreetMap contributors, Overture Maps Foundation}},
  title        = {Overture Maps open data platform},
  year         = {2026},
  howpublished = {\url{https://overturemaps.org}},
  note         = {Accessed: 2026-03-26}
}

@inproceedings{devlin-etal-2019-bert,
    title = "{BERT}: Pre-training of Deep Bidirectional Transformers for Language Understanding",
    author = "Devlin, Jacob  and
      Chang, Ming-Wei  and
      Lee, Kenton  and
      Toutanova, Kristina",
    editor = "Burstein, Jill  and
      Doran, Christy  and
      Solorio, Thamar",
    booktitle = "Proceedings of the 2019 Conference of the North {A}merican Chapter of the Association for Computational Linguistics: Human Language Technologies, Volume 1 (Long and Short Papers)",
    month = jun,
    year = "2019",
    address = "Minneapolis, Minnesota",
    publisher = "Association for Computational Linguistics",
    url = "https://aclanthology.org/N19-1423/",
    doi = "10.18653/v1/N19-1423",
    pages = "4171--4186"
}

@inproceedings{chu2026geo2vec,
  title={Geo2vec: Shape-and distance-aware neural representation of geospatial entities},
  author={Chu, Chen and Shahabi, Cyrus},
  booktitle={Proceedings of the AAAI Conference on Artificial Intelligence},
  volume={40},
  number={23},
  pages={18985--18993},
  year={2026}
}

@inproceedings{li-etal-2022-spabert,
    title = "{S}pa{BERT}: A Pretrained Language Model from Geographic Data for Geo-Entity Representation",
    author = "Li, Zekun  and
      Kim, Jina  and
      Chiang, Yao-Yi  and
      Chen, Muhao",
    editor = "Goldberg, Yoav  and
      Kozareva, Zornitsa  and
      Zhang, Yue",
    booktitle = "Findings of the Association for Computational Linguistics: EMNLP 2022",
    month = dec,
    year = "2022",
    address = "Abu Dhabi, United Arab Emirates",
    publisher = "Association for Computational Linguistics",
    url = "https://aclanthology.org/2022.findings-emnlp.200/",
    doi = "10.18653/v1/2022.findings-emnlp.200",
    pages = "2757--2769"
}

@article{chen2025self,
title = {Self-supervised representation learning for geospatial objects: A survey},
journal = {Information Fusion},
volume = {123},
pages = {103265},
year = {2025},
issn = {1566-2535},
doi = {https://doi.org/10.1016/j.inffus.2025.103265},
url = {https://www.sciencedirect.com/science/article/pii/S1566253525003380},
author = {Yile Chen and Weiming Huang and Kaiqi Zhao and Yue Jiang and Gao Cong}
}

@inproceedings{li2023urban,
author = {Li, Yi and Huang, Weiming and Cong, Gao and Wang, Hao and Wang, Zheng},
title = {Urban Region Representation Learning with OpenStreetMap Building Footprints},
year = {2023},
isbn = {9798400701030},
publisher = {Association for Computing Machinery},
address = {New York, NY, USA},
url = {https://doi.org/10.1145/3580305.3599538},
doi = {10.1145/3580305.3599538},
booktitle = {Proceedings of the 29th ACM SIGKDD Conference on Knowledge Discovery and Data Mining},
pages = {1363–1373},
numpages = {11},
location = {Long Beach, CA, USA},
series = {KDD '23}
}

@article{mai2023towards,
  title={Towards general-purpose representation learning of polygonal geometries},
  author={Mai, Gengchen and Jiang, Chiyu and Sun, Weiwei and Zhu, Rui and Xuan, Yao and Cai, Ling and Janowicz, Krzysztof and Ermon, Stefano and Lao, Ni},
  journal={GeoInformatica},
  volume={27},
  number={2},
  pages={289--340},
  year={2023},
  publisher={Springer}
}

@inproceedings{tempelmeier2021geovectors,
author = {Tempelmeier, Nicolas and Gottschalk, Simon and Demidova, Elena},
title = {GeoVectors: A Linked Open Corpus of OpenStreetMap Embeddings on World Scale},
year = {2021},
isbn = {9781450384469},
publisher = {Association for Computing Machinery},
address = {New York, NY, USA},
url = {https://doi.org/10.1145/3459637.3482004},
doi = {10.1145/3459637.3482004},
booktitle = {Proceedings of the 30th ACM International Conference on Information \& Knowledge Management},
pages = {4604–4612},
numpages = {9},
location = {Virtual Event, Queensland, Australia},
series = {CIKM '21}
}

@inproceedings{wang2019learning,
author = {Wang, Meng-xiang and Lee, Wang-Chien and Fu, Tao-yang and Yu, Ge},
title = {Learning Embeddings of Intersections on Road Networks},
year = {2019},
isbn = {9781450369091},
publisher = {Association for Computing Machinery},
address = {New York, NY, USA},
url = {https://doi.org/10.1145/3347146.3359075},
doi = {10.1145/3347146.3359075},
booktitle = {Proceedings of the 27th ACM SIGSPATIAL International Conference on Advances in Geographic Information Systems},
pages = {309–318},
numpages = {10},
location = {Chicago, IL, USA},
series = {SIGSPATIAL '19}
}

@ARTICLE{jepsen2020relational,
  author={Jepsen, Tobias Skovgaard and Jensen, Christian S. and Nielsen, Thomas Dyhre},
  journal={IEEE Transactions on Intelligent Transportation Systems}, 
  title={Relational Fusion Networks: Graph Convolutional Networks for Road Networks}, 
  year={2022},
  volume={23},
  number={1},
  pages={418-429},
  doi={10.1109/TITS.2020.3011799}}

@inproceedings{grover2016node2vec,
author = {Grover, Aditya and Leskovec, Jure},
title = {node2vec: Scalable Feature Learning for Networks},
year = {2016},
isbn = {9781450342322},
publisher = {Association for Computing Machinery},
address = {New York, NY, USA},
url = {https://doi.org/10.1145/2939672.2939754},
doi = {10.1145/2939672.2939754},
booktitle = {Proceedings of the 22nd ACM SIGKDD International Conference on Knowledge Discovery and Data Mining},
pages = {855–864},
numpages = {10},
location = {San Francisco, California, USA},
series = {KDD '16}
}

@INPROCEEDINGS{hu2019stochastic,
  author={Hu, Jilin and Guo, Chenjuan and Yang, Bin and Jensen, Christian S.},
  booktitle={2019 IEEE 35th International Conference on Data Engineering (ICDE)}, 
  title={Stochastic Weight Completion for Road Networks Using Graph Convolutional Networks}, 
  year={2019},
  volume={},
  number={},
  pages={1274-1285},
  doi={10.1109/ICDE.2019.00116}}

@article{liu2026gair,
  title={GAIR: Location-aware self-supervised contrastive pre-training with geo-aligned implicit representations},
  author={Liu, Zeping and Ni, Lao and Wang, Zhangyu and Jiao, Junfeng and Mai, Gengchen},
  journal={ISPRS Journal of Photogrammetry and Remote Sensing},
  page={166-182},
  volumn={237},
  year={2026}
}

@inproceedings{klemmer2025satclip,
  title={Satclip: Global, general-purpose location embeddings with satellite imagery},
  author={Klemmer, Konstantin and Rolf, Esther and Robinson, Caleb and Mackey, Lester and Ru{\ss}wurm, Marc},
  booktitle={Proceedings of the AAAI Conference on Artificial Intelligence},
  volume={39},
  number={4},
  pages={4347--4355},
  year={2025}
}

@inproceedings{mai2024srl,
  title={SRL: Towards a General-Purpose Framework for Spatial Representation Learning},
  author={ Mai, Gengchen and  Yao, Xiaobai and  Xie, Yiqun and  Rao, Jinmeng and  Li, Hao and  Zhu, Qing and  Li, Ziyuan and  Lao Ni},
  booktitle={32nd ACM SIGSPATIAL International Conference on Advances in Geographic Information Systems},
  year={2024}
}

@article{tobler1970computer,
  title={A computer movie simulating urban growth in the Detroit region},
  author={Tobler, Waldo R},
  journal={Economic geography},
  volume={46},
  number={sup1},
  pages={234--240},
  year={1970},
  publisher={Taylor \& Francis}
}

@article{tobler2004first,
  title={On the first law of geography: A reply},
  author={Tobler, Waldo},
  journal={Annals of the association of American geographers},
  volume={94},
  number={2},
  pages={304--310},
  year={2004},
  publisher={Taylor \& Francis}
}

@article{bachmaier2011variogram,
  title={Variogram or semivariogram? Variance or semivariance? Allan variance or introducing a new term?},
  author={Bachmaier, Martin and Backes, Matthias},
  journal={Mathematical Geosciences},
  volume={43},
  number={6},
  pages={735--740},
  year={2011},
  publisher={Springer}
}

@inproceedings{yan2017itdl,
  title={From itdl to place2vec: Reasoning about place type similarity and relatedness by learning embeddings from augmented spatial contexts},
  author={Yan, Bo and Janowicz, Krzysztof and Mai, Gengchen and Gao, Song},
  booktitle={Proceedings of the 25th ACM SIGSPATIAL international conference on advances in geographic information systems},
  pages={1--10},
  year={2017}
}

@inproceedings{jean2019tile2vec,
  title={Tile2vec: Unsupervised representation learning for spatially distributed data},
  author={Jean, Neal and Wang, Sherrie and Samar, Anshul and Azzari, George and Lobell, David and Ermon, Stefano},
  booktitle={Proceedings of the AAAI Conference on Artificial Intelligence},
  volume={33},
  number={01},
  pages={3967--3974},
  year={2019}
}

@inproceedings{wang2024mc,
  title={MC-GTA: Metric-Constrained Model-Based Clustering using Goodness-of-fit Tests with Autocorrelations},
  author={Wang, Zhangyu and Mai, Gengchen and Janowicz, Krzysztof and Lao, Ni},
  booktitle={41st International Conference on Machine Learning, ICML 2024},
  year={2024}
}

@inproceedings{mai2019relaxing,
  title={Relaxing unanswerable geographic questions using a spatially explicit knowledge graph embedding model},
  author={Mai, Gengchen and Yan, Bo and Janowicz, Krzysztof and Zhu, Rui},
  booktitle={International conference on geographic information science},
  pages={21--39},
  year={2019},
  organization={Springer}
}

@article{huang2022estimating,
  title={Estimating urban functional distributions with semantics preserved POI embedding},
  author={Huang, Weiming and Cui, Lizhen and Chen, Meng and Zhang, Daokun and Yao, Yao},
  journal={International Journal of Geographical Information Science},
  volume={36},
  number={10},
  pages={1905--1930},
  year={2022},
  publisher={Taylor \& Francis}
}

@inproceedings{li2023rethink,
  title={Rethink geographical generalizability with unsupervised self-attention model ensemble: A case study of openstreetmap missing building detection in africa},
  author={Li, Hao and Wang, Jiapan and Zollner, Johann Maximilian and Mai, Gengchen and Lao, Ni and Werner, Martin},
  booktitle={Proceedings of the 31st ACM International Conference on Advances in Geographic Information Systems},
  pages={1--9},
  year={2023}
}

@article{randell1992spatial,
  title={A spatial logic based on regions and connection.},
  author={Randell, David A and Cui, Zhan and Cohn, Anthony G and others},
  journal={KR},
  volume={92},
  number={165-176},
  pages={40--40},
  year={1992}
}

@inproceedings{egenhofer1990mathematical,
  title={A mathematical framework for the definition of topological relations},
  author={Egenhofer, Max},
  booktitle={Proc. the fourth international symposium on spatial data handing},
  pages={803--813},
  year={1990}
}

@article{regalia2019computing,
  title={Computing and querying strict, approximate, and metrically refined topological relations in linked geographic data},
  author={Regalia, Blake and Janowicz, Krzysztof and McKenzie, Grant},
  journal={Transactions in GIS},
  volume={23},
  number={3},
  pages={601--619},
  year={2019},
  publisher={Wiley Online Library}
}

@article{touvron2023llama,
  title={Llama 2: Open foundation and fine-tuned chat models},
  author={Touvron, Hugo and Martin, Louis and Stone, Kevin and Albert, Peter and Almahairi, Amjad and Babaei, Yasmine and Bashlykov, Nikolay and Batra, Soumya and Bhargava, Prajjwal and Bhosale, Shruti and others},
  journal={arXiv preprint arXiv:2307.09288},
  year={2023}
}

@ARTICLE{wan2021tale,
  author={Wan, Huaiyu and Lin, Yan and Guo, Shengnan and Lin, Youfang},
  journal={IEEE Transactions on Knowledge and Data Engineering}, 
  title={Pre-Training Time-Aware Location Embeddings from Spatial-Temporal Trajectories}, 
  year={2022},
  volume={34},
  number={11},
  pages={5510-5523},
  doi={10.1109/TKDE.2021.3057875}}

@inproceedings{zhou2024road,
 author = {Zhou, Haicang and Huang, Weiming and Chen, Yile and He, Tiantian and Cong, Gao and Ong, Yew-Soon},
 booktitle = {Advances in Neural Information Processing Systems},
 doi = {10.52202/079017-0376},
 editor = {A. Globerson and L. Mackey and D. Belgrave and A. Fan and U. Paquet and J. Tomczak and C. Zhang},
 pages = {11789--11813},
 publisher = {Curran Associates, Inc.},
 title = {Road Network Representation Learning with the Third Law of  Geography},
 url = {https://proceedings.neurips.cc/paper_files/paper/2024/file/15cc8e4a46565dab0c1a1220884bd503-Paper-Conference.pdf},
 volume = {37},
 year = {2024}
}

@book{longley2015geographic,
  title={Geographic information science and systems},
  author={Longley, Paul A and Goodchild, Michael F and Maguire, David J and Rhind, David W},
  year={2015},
  publisher={John Wiley \& Sons}
}

@article{zhang2023road,
  title={Road network representation learning: A dual graph-based approach},
  author={Zhang, Liang and Long, Cheng},
  journal={ACM Transactions on Knowledge Discovery from Data},
  volume={17},
  number={9},
  pages={1--25},
  year={2023},
  publisher={ACM New York, NY}
}

@inproceedings{
loshchilov2018decoupled,
title={Decoupled Weight Decay Regularization},
author={Ilya Loshchilov and Frank Hutter},
booktitle={International Conference on Learning Representations},
year={2019},
url={https://openreview.net/forum?id=Bkg6RiCqY7},
}

@article{yang2014modeling,
  title={Modeling user activity preference by leveraging user spatial temporal characteristics in LBSNs},
  author={Yang, Dingqi and Zhang, Daqing and Zheng, Vincent W and Yu, Zhiyong},
  journal={IEEE Transactions on Systems, Man, and Cybernetics: Systems},
  volume={45},
  number={1},
  pages={129--142},
  year={2014},
  publisher={IEEE}
}

@inproceedings{li-etal-2023-geolm,
    title = "{G}eo{LM}: Empowering Language Models for Geospatially Grounded Language Understanding",
    author = "Li, Zekun  and
      Zhou, Wenxuan  and
      Chiang, Yao-Yi  and
      Chen, Muhao",
    editor = "Bouamor, Houda  and
      Pino, Juan  and
      Bali, Kalika",
    booktitle = "Proceedings of the 2023 Conference on Empirical Methods in Natural Language Processing",
    month = dec,
    year = "2023",
    address = "Singapore",
    publisher = "Association for Computational Linguistics",
    url = "https://aclanthology.org/2023.emnlp-main.317/",
    doi = "10.18653/v1/2023.emnlp-main.317",
    pages = "5227--5240"
}

@article{goodchild1992geographical,
  title={Geographical data modeling},
  author={Goodchild, Michael F},
  journal={Computers \& geosciences},
  volume={18},
  number={4},
  pages={401--408},
  year={1992},
  publisher={Elsevier}
}

@inproceedings{ying2021transformers,
 author = {Ying, Chengxuan and Cai, Tianle and Luo, Shengjie and Zheng, Shuxin and Ke, Guolin and He, Di and Shen, Yanming and Liu, Tie-Yan},
 booktitle = {Advances in Neural Information Processing Systems},
 editor = {M. Ranzato and A. Beygelzimer and Y. Dauphin and P.S. Liang and J. Wortman Vaughan},
 pages = {28877--28888},
 publisher = {Curran Associates, Inc.},
 title = {Do Transformers Really Perform Badly for Graph Representation?},
 url = {https://proceedings.neurips.cc/paper_files/paper/2021/file/f1c1592588411002af340cbaedd6fc33-Paper.pdf},
 volume = {34},
 year = {2021}
}

@inproceedings{
zhou2023unimol,
title={Uni-Mol: A Universal 3D Molecular Representation Learning Framework},
author={Gengmo Zhou and Zhifeng Gao and Qiankun Ding and Hang Zheng and Hongteng Xu and Zhewei Wei and Linfeng Zhang and Guolin Ke},
booktitle={The Eleventh International Conference on Learning Representations },
year={2023},
url={https://openreview.net/forum?id=6K2RM6wVqKu}
}
\bibliographystyle{iclr2027_conference}

\newpage
\appendix
\vspace{-1em}
\section{Method Details}\label{appdx:method}
\vspace{-.5em}
\subsection{Topological Relation Definitions}\label{appdx:topo}
\vspace{-.5em}
We adopt a symmetric topology schema $\mathcal{T}=\{\textsc{disjoint},\textsc{contact},\textsc{containment}\}$, constructed from standard intersection and containment predicates derived from DE-9IM~\citep{egenhofer1990mathematical}. The relations are mutually exclusive under the precedence $\textsc{CONTAINMENT}>\textsc{CONTACT}>\textsc{DISJOINT}$:
\begin{itemize}[leftmargin=*,topsep=0pt,itemsep=1pt,partopsep=0pt,parsep=0pt]
    \item \textsc{disjoint}: $g_i\cap g_j=\varnothing$.
    \item \textsc{contact}: the geometry intersect, but neither is within the other. This relation includes both boundary contact and interior intersection.
    \item \textsc{containment}: one geometry is within the other. We map \textsc{within} and \textsc{contains} to the same relation.
\end{itemize}
The labels are symmetric: $(v_i,v_j)$ and $(v_j,v_i)$ receive the same relation. This design represents shared spatial configuration rather than directional geometric roles. For example, a shop within a mall and the mall enclosing the shop describe the same spatial configuration from complementary perspectives. This design aligns with topology-modulated distance bias and geometry-aware spatial relation modeling. 
\vspace{-.5em}
\subsection{Distance Normalization and Radial Basis Features}
\label{appdx:rbf}
\vspace{-.5em}
The proposed topology-modulated distance bias encodes the minimum distance between geometry with $K=16$ Gaussian radial basis functions. We expand distance into overlapping Gaussian radial basis features so that a learned weighted sum can express smooth, nonlinear distance effects on attention, rather than restricting the bias to a linear function of distance. In \modelname, distances are computed after the fixed affine normalization to $[-1,1]$ applied per city. Let $\bar d_{ij}$ denote this normalized distance and let $D_{\max}$ be the largest normalized distance retained by the dense relation representation ($D_{\max}=0.105$ for NYC and $D_{\max}=0.144$ for Singapore). The centers and shared width are
\begin{equation}
\mu_k=\frac{kD_{\max}}{K-1},
\qquad
\sigma=\frac{D_{\max}}{K-1},
\qquad k=0,\ldots,K-1,
\end{equation}
and the $k$-th feature is
\begin{equation}
\phi_k(\bar d_{ij})
=
\exp\!\left[
-\frac{1}{2}
\left(
\frac{\operatorname{clip}(\bar d_{ij},0,D_{\max})-\mu_k}{\sigma}
\right)^2
\right].
\end{equation}
Thus $\boldsymbol\phi(\bar d_{ij})=[\phi_0(\bar d_{ij}),\ldots,\phi_{K-1}(\bar d_{ij})]$. 

\vspace{-.5em}
\subsection{Anchor and Sibling-Group Construction}\label{appdx:anchor}
\vspace{-.5em}
We construct three reference-relation spaces:
\begin{enumerate}[leftmargin=*,topsep=0pt,itemsep=1pt,partopsep=0pt,parsep=0pt]
    \item polyline anchors and polygon members within a 30-meter buffer (e.g., buildings near roads)
    \item polyline anchors and point members within a 30-meter buffer (e.g., POIs near roads)
    \item polygon anchors and point members within a 5-meter buffer to capture containment (e.g., POIs inside buildings)
\end{enumerate}
Points and polygons serve as sibling members, whereas polylines serve only as anchors. Common vector datasets provide no general relation that groups multiple polylines under a shared reference entity without task-specific semantics (e.g., road segments sharing a transit route); we therefore do not treat polylines as members.

These thresholds are evaluated with \texttt{ST\_DWithin} on complete geometry in EPSG:3857. A geoentity may belong to the sibling group, but anchor entities are not added to a context solely because they define a group. For the non-sibling baseline used by $\mathcal L_{\mathrm{RSR}}$, we retain same-type point--point and polygon--polygon base-context pairs whose complete-geometry distance is at most $50$ projected meters, excluding every pair that shares any sibling group.
\vspace{-.5em}
\subsection{Relational Semivariogram Details}\label{appdx:semivariogram}
\vspace{-.5em}
The relation-conditioned semivariogram regularizer compares sibling and non-sibling dispersion at similar distances. We formalize the connection between cosine distance and squared Euclidean distance under $\ell_2$ normalization, and show that enforcing lower semivariance translates to a margin-based similarity gap. 

\begin{proposition}[Equivalence under normalization]
\label{prop:cosine_l2}
If $\|\mathbf h_i^{\mathrm{sem}}\|_2=\|\mathbf h_j^{\mathrm{sem}}\|_2=1$, then
\begin{equation}
\frac{1}{2}
\left\|
\mathbf h_i^{\mathrm{sem}}-\mathbf h_j^{\mathrm{sem}}
\right\|_2^2
=
1-\operatorname{sim}
\!\left(
\mathbf h_i^{\mathrm{sem}},
\mathbf h_j^{\mathrm{sem}}
\right).
\end{equation}
\end{proposition}
\begin{proof}
Expanding the squared distance gives
\begin{align}
\left\|\mathbf h_i^{\mathrm{sem}}-\mathbf h_j^{\mathrm{sem}}\right\|_2^2
&=
\left\|\mathbf h_i^{\mathrm{sem}}\right\|_2^2
+
\left\|\mathbf h_j^{\mathrm{sem}}\right\|_2^2
-2(\mathbf h_i^{\mathrm{sem}})^{\top}\mathbf h_j^{\mathrm{sem}} \\
&=
2\left[
1-\operatorname{sim}
\!\left(
\mathbf h_i^{\mathrm{sem}},
\mathbf h_j^{\mathrm{sem}}
\right)
\right],
\end{align}
since $\|\mathbf h_i^{\mathrm{sem}}\|_2=\|\mathbf h_j^{\mathrm{sem}}\|_2=1$. Dividing by two proves the result.
\renewcommand{\qedsymbol}{}
\end{proof}

Let $\|\mathbf{h}_{i,\ell}^{\mathrm{sem}}-\mathbf{h}_{j,\ell}^{\mathrm{sem}}\|^2=2(1-c_{ij})$. For a pair set $\Pi$ in context $\ell$, define
\begin{equation}
\hat{\gamma}(\Pi)=1-\bar{c}(\Pi),
\qquad
\bar{c}(\Pi)=\frac{1}{|\Pi|}\sum_{(v_i,v_j)\in\Pi}c_{ij},
\label{eq:semivariance_similarity}
\end{equation}
where $\bar{c}(\Pi)$ is the average cosine similarity over $\Pi$. By \Cref{prop:cosine_l2}, this is the empirical semivariance of the normalized contextual semantic embeddings.

For a sibling group $\mathcal{S}_{a,r}$ and bin $b$, let $\Pi_{\mathrm{rel}}^{(a,r)}(b)$ denote its sibling pairs in bin $b$ and $\Pi_{\mathrm{glob}}^{(\eta_{a,r})}(b)$ the corresponding type-matched non-sibling pairs in the same context, so that
\begin{equation}
\gamma_{\mathrm{rel}}^{(a,r)}(b)=\hat{\gamma}\big(\Pi_{\mathrm{rel}}^{(a,r)}(b)\big),
\qquad
\gamma_{\mathrm{glob}}^{(\eta_{a,r})}(b)=\hat{\gamma}\big(\Pi_{\mathrm{glob}}^{(\eta_{a,r})}(b)\big).
\label{eq:conditional_semivariance}
\end{equation}

\begin{proposition}[Dispersion ordering implies a similarity gap]
\label{prop:relvar_gap}
Fix a sibling group $\mathcal{S}_{a,r}$ and a valid bin $b$. The corresponding hinge term in~\Cref{eq:rsr_loss} is zero if and only if
\begin{equation}
\gamma_{\mathrm{rel}}^{(a,r)}(b)
\le
\gamma_{\mathrm{glob}}^{(\eta_{a,r})}(b)-\delta,
\label{eq:rsr_dispersion_order}
\end{equation}
or, equivalently,
\begin{equation}
\bar{c}\big(\Pi_{\mathrm{rel}}^{(a,r)}(b)\big)
\ge
\bar{c}\big(\Pi_{\mathrm{glob}}^{(\eta_{a,r})}(b)\big)+\delta.
\label{eq:rsr_similarity_gap}
\end{equation}
\end{proposition}
\begin{proof}
The hinge term $\max\!\big(0,\,\gamma_{\mathrm{rel}}^{(a,r)}(b)-\gamma_{\mathrm{glob}}^{(\eta_{a,r})}(b)+\delta\big)$ is zero if and only if $\gamma_{\mathrm{rel}}^{(a,r)}(b)-\gamma_{\mathrm{glob}}^{(\eta_{a,r})}(b)+\delta\le 0$, which is~\Cref{eq:rsr_dispersion_order}. Substituting $\hat{\gamma}(\Pi)=1-\bar{c}(\Pi)$ from~\Cref{eq:semivariance_similarity} into both sides gives
\begin{equation}
1-\bar{c}\big(\Pi_{\mathrm{rel}}^{(a,r)}(b)\big)
\le
1-\bar{c}\big(\Pi_{\mathrm{glob}}^{(\eta_{a,r})}(b)\big)-\delta,
\end{equation}
which rearranges to~\Cref{eq:rsr_similarity_gap}.
\renewcommand{\qedsymbol}{}
\end{proof}
\paragraph{Interpretation.} \Cref{prop:relvar_gap} shows that minimizing $\mathcal{L}_{\text{RSR}}$ enforces a margin-based similarity gap between sibling and non-sibling pairs at each distance bin: sibling pairs must be at least $\delta$ more similar (in cosine similarity) than type-matched non-sibling pairs at the same distance, ensuring that relational structure enhances semantic coherence beyond metric proximity alone.

\modelname uses five equal-width bins over $[0,50]$ projected meters of geometry-to-geometry distance, matching the range of the non-sibling pairs (\appxref{appdx:anchor}). A bin is valid for $\mathcal{S}_{a,r}$ if both pair sets contain more than five pairs; only valid bins contribute to~\Cref{eq:rsr_loss}. For valid bins, the weights are
\begin{equation}
\pi_b^{(a,r)}
=
\frac{\big|\Pi_{\mathrm{rel}}^{(a,r)}(b)\big|}
{\sum_{b'}\big|\Pi_{\mathrm{rel}}^{(a,r)}(b')\big|},
\label{eq:rsr_bin_weights}
\end{equation}
where $b'$ ranges over the valid bins of $\mathcal{S}_{a,r}$, so that the weights of each group sum to one.

\vspace{-1em}
\section{Experimental Setup}\label{appdx:setup}
\vspace{-.5em}
\subsection{\modelname Implementation Details}\label{appdx:implementation}
\vspace{-.5em}
We encode semantic attributes with a frozen \texttt{bert-base-uncased} model using mean pooling, producing a normalized 768-dimensional embedding. Complete point, polyline, and polygon geometry are converted to 210-dimensional complex Fourier features and passed to Poly2Vec~\citep{siampou2025polyvec}, whose hidden and output dimensions are 256 and 128. The semantic and geometric representations are projected to $d_{\mathrm{model}}=128$ and combined by the adaptive gate in~\Cref{eq:fused_emb}; the final gate bias is initialized to $-2.0$.

The contextual encoder has three Transformer layers, four attention heads, a feed-forward dimension of 256, and dropout 0.1. Queries and keys are shared and are computed from the fused embeddings. The semantic and fused embeddings use separate value and output projections, residual connections, layer normalization, and feed-forward networks. The topology-modulated distance bias uses the 16 radial basis features defined in~\appxref{appdx:rbf}; its topology intercepts and distance weights are initialized to zero. We set the masking ratio to $40\%$. 

We pretrain \modelname for 100 epochs with AdamW~\citep{loshchilov2018decoupled}, learning rate $2\times10^{-4}$, and weight decay 0.01. Training uses seed 42 and an optimizer batch of 256 contexts. The finalized loss is
\begin{equation}
\mathcal L
=
\mathcal L_{\mathrm{MGSM}}
+
0.5\mathcal L_{\mathrm{topo}}
+
200\mathcal L_{\mathrm{dist}}
+
\mathcal L_{\mathrm{ACC}}
+
25\mathcal L_{\mathrm{RSR}}.
\end{equation}
We set $\tau_{\mathrm{MGSM}}=0.15$, $\tau_{\mathrm{ACC}}=0.4$, and use the ACC weight
$\exp(-d_{ij}/20)$, where $d_{ij}$ is the complete-geometry distance in projected meters. For RSR, we use margin $\delta=0.1$, five $10$-meter intervals over $[0,50]$ meters. Their different magnitudes compensate for differences in the numerical scales of the individual objectives and should not be interpreted as relative semantic importance. We share all pretraining hyperparameters across NYC and Singapore. Pretraining \modelname requires approximately 7 hours for NYC and 4 hours for Singapore on a single NVIDIA RTX A5000 GPU. 

\vspace{-.5em}
\subsection{Pretraining Dataset Statistics}\label{appdx:pretraining}
\vspace{-.5em}
Polylines are noded at intersections before pretraining, and each noded entity remains a separate model token and reference entity. The finalized $250$-meter construction contains 87,135 spatial contexts for NYC and 44,311 for Singapore. 

\begin{table}[h]
\centering
\small
\caption{Pretraining geoentities and context statistics after preprocessing and one-hop reference-relation expansion. Context means are computed over nonempty contexts.}
\label{tab:pretraining_stats}
\centering
\small
\begin{tabular}{lrr}
\toprule
& NYC & Singapore \\
\midrule
Point geoentities & 43,549 & 23,621 \\
Polygon geoentities & 49,486 & 40,594 \\
Polyline geoentities & 186,667 & 181,592 \\
\midrule
Total \# of contexts & 87,135 & 44,311 \\
Mean entities per context & 25.213 & 38.133 \\
Mean mask-eligible entities & 3.210 & 5.547 \\
Maximum context size & 573 & 498 \\
\bottomrule
\end{tabular}
\vspace{-.5em}
\end{table}

\vspace{-.5em}
\subsection{Downstream Task Protocols}\label{appdx:protocols}
No downstream labels are used during \modelname pretraining. Target road geoentities for traffic-speed estimation may appear in the unlabeled OSM pretraining data, but their Uber Movement traffic speeds are never used during pretraining. Untagged target buildings are excluded from pretraining following the CityFM protocol~\citep{balsebre2024city}, and target Foursquare POIs are absent from the OSM pretraining data. Building function classification and next POI prediction therefore additionally evaluate geoentities unseen during pretraining.

\vspace{-.5em} \paragraph{Task A: Traffic Speed Estimation.}
We use the New York City traffic-speed benchmark introduced by CityFM~\citep{balsebre2024city}. The targets are Uber Movement speed measurements, in miles per hour (mph), mapped to OSM parent-road segments. Roads with fewer than 10 observations are excluded, leaving 29,755 road segments (mean speed 19.50 mph, standard deviation 8.69 mph, and range 1.46--57.37 mph). Following the CityFM protocol~\citep{balsebre2024city}, we use one partition of 17,853 training, 5,951 validation, and 5,951 test roads, corresponding to a 60/20/20 split. For each seed, model selection uses only the validation set, and the selected checkpoint is evaluated once on the test set. Evaluation metrics include root mean square error (RMSE), mean absolute error (MAE), coefficient of determination (R$^2$), and mean absolute percentage error (MAPE), reported as mean and standard deviation over 10 seeds (42--51).

\vspace{-.5em}\paragraph{Task B: Building Function Classification.}
We use the Singapore building-function benchmark introduced by CityFM~\citep{balsebre2024city}, consisting of 64,384 OSM building polygons labeled according to an eight-class land-use taxonomy. Following the CityFM protocol~\citep{balsebre2024city}, we use the 10 folds, each containing 38,628 training, 12,876 validation, and 12,880 test buildings, corresponding approximately to a 60/20/20 split. Within each fold, the checkpoint with the highest validation Macro-F1 is selected and evaluated once on the held-out test set. Evaluation metrics include Accuracy, Macro-F1, and Weighted-F1, reported as mean and standard deviation over 10 folds. In the released CityFM evaluation code, Weighted-F1 is computed as \texttt{f1\_score(y\_pred, y\_true, average='weighted')}, passing arguments in the order \texttt{(y\_pred, y\_true)}. Because \texttt{sklearn}'s weighted averaging draws class support from the first argument, this weights per-class F1 scores by predicted class frequencies rather than ground-truth class frequencies. We recompute Weighted-F1 identically for every evaluated method, CityFM included, as \texttt{f1\_score(y\_true, y\_pred, labels=range(8), average='weighted', zero\_division=0)}. This choice affects only Weighted-F1: Macro-F1 is unchanged, since each per-class F1 and its unweighted average are invariant to exchanging precision and recall, and Accuracy is order-independent. Published CityFM Weighted-F1 values therefore reflect a different weighting and are not directly comparable to those reported here.

\vspace{-.5em}\paragraph{Task C: Next POI Prediction.} 
We use the Foursquare check-in datasets from~\citet{yang2014modeling} following the task formulation of POI-Enhancer~\citep{cheng2025poi}. Foursquare-NY contains 15,171 users, 24,118 POIs, and 641,005 check-ins, and Foursquare-SG contains 10,909 users, 20,154 POIs, and 696,306 check-ins. Our protocol uses fixed 80/10/10 train/validation/test trajectory partitions with split seed 42, producing 91,352/11,419/11,420 trajectories for New York and 110,810/13,851/13,852 trajectories for Singapore. For each configuration, we run downstream training with seeds 42, 43, and 44 for at most 30 epochs. Validation Hit@1 is measured after every epoch, and we select, independently per seed, the earliest epoch that achieves the maximum validation Hit@1. The test set is evaluated once using that checkpoint. Hit@1 and Hit@5 are taken from the same selected epoch and reported as mean and standard deviation over the three seeds (42--44).

\vspace{-.5em}
\subsection{Baseline Details}\label{appdx:baseline}
\vspace{-.5em} \paragraph{Task A: Traffic Speed Prediction Baselines.}
\begin{itemize}[leftmargin=*,topsep=0pt,itemsep=1pt,partopsep=0pt,parsep=0pt]

\item \textit{Topology (Road Network):}
\begin{itemize}
    \item \textbf{Node2Vec}~\citep{grover2016node2vec}: learns node representations from biased random walks over the road-network graph by encouraging neighboring nodes in the walk context to have similar embeddings.
    
    \item \textbf{GCWC}~\citep{hu2019stochastic}: a graph convolutional method that captures edge correlations and local spatial dependencies in road networks.
    
    \item \textbf{RFN}~\citep{jepsen2020relational}: a graph neural network that models relational dependencies between connected road segments.
\end{itemize}

\item \textit{Geometry:}
\begin{itemize}
    \item \textbf{Poly2Vec}~\citep{siampou2025polyvec}: a geometry encoder that represents coordinate sequences in the spectral domain using Fourier transforms.
    
    \item \textbf{Geo2Vec}~\citep{chu2026geo2vec}: a geometry encoder based on Signed Distance Fields (SDFs) that captures fine-grained spatial boundary information.
\end{itemize}

\item \textit{Semantics + Metric:}
\begin{itemize}
    \item \textbf{IRN2Vec}~\citep{wang2019learning}: learns road-intersection embeddings from random-walk sequences generated over adjacent intersections.
    
    \item \textbf{GeoVectors}~\citep{tempelmeier2021geovectors}: learns embeddings of OSM entities by jointly modeling geographic coordinates and semantic tags.
\end{itemize}

\item \textit{Semantics + Geometry:}
\begin{itemize}
    \item \textbf{CityFM}~\citep{balsebre2024city}: combines hand-crafted polyline features, positional encodings, and learned semantic representations of surrounding point and polygon geoentities.
\end{itemize}

\end{itemize}

\vspace{-.5em}
\paragraph{Task B: Building Function Classification Baselines.}
\begin{itemize}[leftmargin=*,topsep=0pt,itemsep=1pt,partopsep=0pt,parsep=0pt]

\item \textit{Non-Spatial Semantics:}
\begin{itemize}
    \item \textbf{BERT (fine-tuned)}~\citep{devlin-etal-2019-bert}: a pretrained language model used as the textual encoder backbone.
\end{itemize}

\item \textit{Geometry:}
\begin{itemize}
    \item \textbf{Poly2Vec}~\citep{siampou2025polyvec}, \textbf{Geo2Vec}~\citep{chu2026geo2vec}: geometry encoders that map spatial coordinates into continuous latent representations.
\end{itemize}

\item \textit{Semantics + Metric:}
\begin{itemize}
    \item \textbf{GeoVectors}~\citep{tempelmeier2021geovectors}: encodes buildings using geographic coordinates and semantic information from neighboring entities.
    
    \item \textbf{SpaBERT}~\citep{li-etal-2022-spabert}: encodes surrounding entity names together with relative distance information from the target building.
\end{itemize}

\item \textit{Semantics + Geometry:}
\begin{itemize}
    \item \textbf{CityFM}~\citep{balsebre2024city}: combines polygon geometry features, positional encodings, polygon area, and learned semantic representations of nearby geographic entities.
\end{itemize}
\end{itemize}

\vspace{-.5em}
\paragraph{Task C: Next POI Prediction Baselines.}
\begin{itemize}[leftmargin=*,topsep=0pt,itemsep=1pt,partopsep=0pt,parsep=0pt]

\item \textbf{TALE}~\citep{wan2021tale}: learns POI embeddings from visit co-occurrence patterns while modeling temporal periodicity through hierarchical time discretization.

\item \textbf{CTLE}~\citep{lin2021pre}: a self-supervised Transformer that learns contextual POI representations by reconstructing masked location and temporal attributes from visit sequences.

\item \textbf{POI-Enhancer}~\citep{cheng2025poi}: enhances POI representations using auxiliary textual knowledge generated by large language models (e.g., Llama-2-7B~\citep{touvron2023llama}). It incorporates information such as addresses, nearby POIs, and check-in statistics, and aligns text-enhanced representations with pretrained POI embeddings (e.g., TALE or CTLE).
\end{itemize}

\vspace{-1em}
\section{Additional Results}\label{appdx:analyses}
\vspace{-.5em}

\subsection{Polyline analyses: performance by road type}\label{appdx:polyline}
\vspace{-.5em}
\Cref{tab:road_type_analysis} shows per-type MAE on the NYC test set. \modelname reduces MAE relative to CityFM in 13 of 16 road types, including all four most frequent classes, namely residential (5.4\%), primary (6.4\%), secondary (9.9\%), and tertiary (8.8\%) roads. These four classes contain 83\% of test roads and account for 67\% of the total reduction in absolute error. Gains are also large on high-speed classes such as trunk roads (15.9\%), motorway links (13.8\%), and motorways (11.9\%); motorway links alone contribute 16\% of the total reduction. The three types where \modelname underperforms (construction, living street, and service) together contain only 20 test roads (0.3\%). Without CityFM's engineered road features, \modelname$^\dagger$ still outperforms CityFM on all four most frequent classes, and the engineered features mainly help on high-speed classes such as motorway links, trunk roads, and primary links.

\begin{table}[h]
\vspace{-.5em}
\centering
\caption{Per-type MAE for traffic speed prediction on the NYC test set, using representative downstream runs. \modelname$^\dagger$ excludes CityFM's engineered road features. $\Delta$MAE $=$ MAE(CityFM) $-$ MAE(\modelname); positive values denote lower error than CityFM.}
\label{tab:road_type_analysis}
\scalebox{0.82}{%
\begin{tabular}{lrrrrr}
\toprule
Road type & Count & CityFM & NARA$^\dagger$ & NARA & $\Delta$MAE \\
\midrule
Pedestrian       & 1     & 2.605 & 0.852 & 0.030 & +2.575 \\
Tertiary Link    & 12    & 3.946 & 1.876 & 2.606 & +1.340 \\
Trunk Link       & 30    & 4.613 & 4.162 & 3.496 & +1.117 \\
Secondary Link   & 37    & 3.897 & 3.234 & 3.123 & +0.775 \\
Motorway Link    & 377   & 4.942 & 5.007 & 4.262 & +0.680 \\
Trunk             & 71    & 3.959 & 3.987 & 3.331 & +0.628 \\
Motorway          & 279   & 3.768 & 3.642 & 3.320 & +0.449 \\
Secondary         & 1,100 & 3.182 & 2.922 & 2.868 & +0.315 \\
Primary Link      & 59    & 3.907 & 4.068 & 3.610 & +0.297 \\
Tertiary          & 602   & 3.202 & 2.941 & 2.921 & +0.281 \\
Unclassified      & 89    & 3.694 & 3.484 & 3.474 & +0.220 \\
Primary           & 1,114 & 3.117 & 3.025 & 2.918 & +0.199 \\
Residential       & 2,160 & 2.767 & 2.651 & 2.618 & +0.150 \\
Construction      & 4     & 2.216 & 3.360 & 2.689 & -0.473 \\
Living Street     & 4     & 3.343 & 2.893 & 3.907 & -0.564 \\
Service           & 12    & 3.629 & 3.477 & 4.711 & -1.082 \\
\bottomrule
\end{tabular}}
\vspace{-1em}
\end{table}

\vspace{-.5em}
\subsection{Polygon Analyses: performance by land use label}\label{appdx:polygon}
\vspace{-.5em}
\Cref{tab:building_per_class} evaluates performance across land-use categories. \modelname outperforms CityFM in $5$ of $8$ classes, with particularly large gains for Civic \& Community Institutions (+27.49 F1), Commercial \& Residential (+20.61 F1), and Educational buildings (+16.74 F1), suggesting that relational spatial context is especially valuable for less frequent or functionally ambiguous building types. Gains are not uniform: \modelname underperforms on Transport (-18.39) and Sports \& Recreation (-8.05), low-support classes whose footprints limit reliable contextualization.

\begin{table}[h]
\vspace{-.5em}
\centering
\caption{Per-class F1 score for building function classification. Counts denote the number of test instances per class.}
\label{tab:building_per_class}
\scalebox{0.8}{
\begin{tabular}{lrrrr}
\toprule
Category & Count & CityFM & \modelname & $\Delta$F1 \\
\midrule
Residential                 & 8,645 & 96.03 & 96.77 & +0.74 \\
Industrial                  & 2,087 & 94.65 & 92.40 & -2.25 \\
Commercial                  & 1,038 & 87.07 & 90.23 & +3.16 \\
Commercial \& Residential   &   329 & 46.29 & 66.90 & +20.61 \\
Educational                 &   286 & 73.00 & 89.74 & +16.74 \\
Civic \& Community Inst.    &   241 & 24.92 & 52.41 & +27.49 \\
Sports \& Recreation        &   151 & 69.76 & 61.71 & -8.05 \\
Transport                   &   103 & 61.53 & 43.14 & -18.39 \\
\bottomrule
\end{tabular}}
\end{table}
\vspace{-.5em}
\subsection{Point analyses}
\vspace{-.5em}
\subsubsection{Performance by check-in POI type}\label{appdx:point_type}
We evaluate next POI prediction across the $15$ most frequent venue types in the Foursquare-New York test set (\Cref{tab:poi_per_type}). \modelname improves Hit@1 in $13$ categories and ties in the remaining two, while improving Hit@5 in all $15$. The largest Hit@1 gains occur for Deli / Bodega (+3.01 points), Home (private) (+1.80), and transit- or neighborhood-linked venues such as Subway (+1.62), Park (+1.61), and Neighborhood (+1.47), consistent with NARA benefiting from the surrounding OSM spatial context.

\begin{table}[h]
\centering
\caption{Per-category Hit@1 and Hit@5 for next POI prediction (NY-CTLE). Counts denote test prediction events.}
\label{tab:poi_per_type}
\scalebox{0.82}{
\begin{tabular}{lccccc}
\toprule
\multirow{2}{*}{Foursquare POI Type} & \multirow{2}{*}{Count} 
& \multicolumn{2}{c}{POI-Enhancer} 
& \multicolumn{2}{c}{\modelname} \\
\cmidrule(lr){3-4} \cmidrule(lr){5-6}
& & Hit@1 & Hit@5 & Hit@1 & Hit@5 \\
\midrule
Subway              & 927 & 16.757 & 34.808 & 18.375 & 38.799 \\
Bar                 & 891 & 1.534  & 3.330  & 1.534           & 4.003  \\
Home (private)      & 816 & 17.198 & 32.598 & 18.995 & 38.766 \\
Office              & 689 & 20.077 & 32.898 & 20.174 & 35.897 \\
Coffee Shop         & 649 & 4.212  & 8.783  & 4.417  & 10.324 \\
Park                & 559 & 6.142  & 12.403 & 7.752 & 13.715 \\
Neighborhood        & 477 & 6.499  & 16.492 & 7.966  & 20.266 \\
Train Station       & 442 & 13.047 & 25.792 & 14.178 & 30.769 \\
Hotel               & 422 & 6.635  & 11.848 & 7.583  & 15.324 \\
American Restaurant & 400 & 2.417  & 6.833  & 2.750  & 8.333 \\
Grocery Store       & 378 & 4.145  & 12.169 & 4.586  & 15.168 \\
Deli / Bodega       & 355 & 13.333 & 31.455 & 16.338 & 34.178 \\
Building            & 336 & 12.302 & 21.627 & 13.492 & 27.480 \\
Italian Restaurant  & 309 & 2.050  & 3.991  & 2.050& 5.394  \\
Plaza               & 296 & 1.014  & 6.306  & 1.914  & 8.108  \\
\bottomrule
\end{tabular}}
\end{table}

\subsubsection{Robustness Across Prediction Conditions}\label{appdx:point_slice} 
Performance on next POI prediction varies widely across conditions for every model. For example, baseline Hit@5 for NY-CTLE ranges from 6.9\% to 30.1\% across the slices below, so the overall gain does not show whether a method helps in all conditions. We therefore split test events by one factor at a time (\Cref{tab:poi_slices}). The five factors are the number of geoentities in the target's context (neighborhood density), the fraction of these geoentities that share an anchor relation with the target (sibling share), how often the target POI is visited (popularity), the number of prior check-ins (1, 2--3, or 4+), and whether the target POI is new or revisited. Slice boundaries are fixed per city before evaluation, and every slice has at least 1.6k events. Results are averaged over three downstream seeds, and 95\% confidence intervals are computed by resampling test trajectories. 

Across the 14 slices and the four combinations of city (NYC/SG) and trajectory model (CTLE~\citep{lin2021pre}/TALE~\citep{wan2021tale}) settings, \modelname improves over POI-Enhancer in 54 of 56 comparisons on Hit@1 and 55 of 56 on Hit@5, and the gain is statistically significant in 49 and 51 of them, respectively. Gains grow with neighborhood density in Singapore, where the Hit@1 change rises from a 0.1\% decrease to a 9.7\% gain for CTLE and from 10.1\% to 27.4\% for TALE. Gains are similar across sibling-share levels, from 10.7\% to 16.8\% in Hit@1 for NY-CTLE, so \modelname does not need targets with many siblings. Gains also hold for targets with only one prior check-in in every setting and for rare targets in three of four settings, and popular targets gain up to 24.5\% (SG-TALE). All three comparisons without improvement are under SG-CTLE, the setting with the smallest overall gain (3.3\%). These are sparse neighborhoods on Hit@1 and rare targets on Hit@1 and Hit@5, and none of these decreases is statistically significant.

\begin{table}[h]
\vspace{-1em}
\centering
\small
\setlength{\tabcolsep}{3.5pt}
\caption{Relative improvement (\%) in Hit@1 and Hit@5 of \modelname over POI-Enhancer~\citep{cheng2025poi} on next POI prediction, averaged over three seeds. Each row is a group of test events. For neighborhood density, sibling share, and target popularity, events are divided into three equal-sized groups (low, mid, high). History length is the number of prior check-ins. Group boundaries are fixed per city before evaluation. Bold values are statistically significant (95\% bootstrap CI excludes zero). $^\dagger$Sibling-share groups are formed within each density group, so that sibling share is not confounded with density.}
\label{tab:poi_slices}
\scalebox{0.82}{\begin{tabular}{llr rr rr rr rr}
\toprule
 & & & \multicolumn{2}{c}{NY-CTLE} & \multicolumn{2}{c}{NY-TALE} & \multicolumn{2}{c}{SG-CTLE} & \multicolumn{2}{c}{SG-TALE} \\
\cmidrule(lr){4-5}\cmidrule(lr){6-7}\cmidrule(lr){8-9}\cmidrule(lr){10-11}
Factor & Slice & $n$ (NY/SG) & H@1 & H@5 & H@1 & H@5 & H@1 & H@5 & H@1 & H@5 \\
\midrule
\multicolumn{2}{l}{All events} & 21.9k/31.9k & \textbf{+14.1} & \textbf{+17.4} & \textbf{+5.8} & \textbf{+7.5} & \textbf{+3.3} & \textbf{+3.3} & \textbf{+13.6} & \textbf{+15.2} \\
\midrule
Density & Sparse & 7.5k/10.8k & \textbf{+13.8} & \textbf{+17.5} & \textbf{+5.3} & \textbf{+9.1} & $-$0.1 & +0.8 & \textbf{+10.1} & \textbf{+12.8} \\
 & Medium & 7.1k/10.5k & \textbf{+12.2} & \textbf{+17.8} & \textbf{+9.3} & \textbf{+6.7} & \textbf{+5.3} & \textbf{+3.9} & \textbf{+11.2} & \textbf{+13.0} \\
 & Dense & 7.3k/10.6k & \textbf{+17.5} & \textbf{+16.6} & +2.5 & \textbf{+4.9} & \textbf{+9.7} & \textbf{+7.0} & \textbf{+27.4} & \textbf{+22.6} \\
\midrule
Sibling share$^\dagger$ & Low & 7.4k/10.7k & \textbf{+16.8} & \textbf{+16.4} & \textbf{+5.5} & \textbf{+7.9} & \textbf{+3.0} & \textbf{+3.0} & \textbf{+15.2} & \textbf{+14.8} \\
 & Mid & 7.2k/10.7k & \textbf{+10.7} & \textbf{+17.4} & \textbf{+7.4} & \textbf{+7.3} & \textbf{+4.0} & \textbf{+5.5} & \textbf{+10.9} & \textbf{+14.8} \\
 & High & 7.3k/10.5k & \textbf{+14.6} & \textbf{+18.7} & \textbf{+4.5} & \textbf{+7.2} & \textbf{+3.1} & +1.5 & \textbf{+14.6} & \textbf{+15.9} \\
\midrule
Prefix length & 1 & 11.4k/13.9k & \textbf{+10.6} & \textbf{+15.3} & \textbf{+3.8} & \textbf{+5.7} & \textbf{+3.9} & +1.0 & \textbf{+12.5} & \textbf{+13.5} \\
 & 2--3 & 6.4k/10.2k & \textbf{+21.0} & \textbf{+24.3} & \textbf{+9.8} & \textbf{+6.9} & \textbf{+3.4} & \textbf{+4.9} & \textbf{+15.4} & \textbf{+17.4} \\
 & 4+ & 4.1k/7.8k & \textbf{+12.9} & \textbf{+14.2} & \textbf{+5.0} & \textbf{+10.8} & \textbf{+2.9} & \textbf{+4.0} & \textbf{+13.2} & \textbf{+14.9} \\
\midrule
Popularity & Rare & 7.5k/11.9k & \textbf{+16.0} & \textbf{+22.1} & +0.6 & \textbf{+12.5} & $-$1.6 & $-$0.1 & \textbf{+11.3} & \textbf{+15.3} \\
 & Mid & 7.8k/9.4k & \textbf{+12.2} & \textbf{+19.5} & \textbf{+6.4} & \textbf{+7.2} & +1.4 & +1.2 & \textbf{+6.4} & \textbf{+12.4} \\
 & Popular & 6.6k/10.6k & \textbf{+14.9} & \textbf{+14.1} & \textbf{+7.5} & \textbf{+5.9} & \textbf{+9.9} & \textbf{+7.3} & \textbf{+24.5} & \textbf{+17.6} \\
\midrule
Novelty & New POI & 20.3k/28.6k & \textbf{+14.2} & \textbf{+17.7} & \textbf{+6.4} & \textbf{+8.1} & \textbf{+3.7} & \textbf{+3.5} & \textbf{+13.8} & \textbf{+15.5} \\
 & Revisit & 1.6k/3.3k & \textbf{+13.6} & \textbf{+15.7} & +2.6 & \textbf{+4.3} & +2.0 & \textbf{+2.1} & \textbf{+13.1} & \textbf{+14.0} \\
\bottomrule
\end{tabular}}
\end{table}

\subsection{Additional Ablations}\label{appdx:ablations_bias}
\vspace{-.5em}
We compare variants of the attention bias that differ in how they combine distance and topology, keeping all other components of \modelname fixed (\Cref{tab:bias_ablation}). Topology is more informative than distance, as topology alone raises Macro-F1 from 70.92 to 72.73, whereas distance alone brings no gain. Adding distance and topology independently, with distance weights shared across topologies, raises Macro-F1 to 73.68. Modulating the distance effect by topology, as in \modelname, performs best on all three metrics, consistent with the principle in \Cref{sec:theory_spatial_dependencies} that the same distance carries different contextual meaning under different topological relations.

\begin{table}[h]
\vspace{-.5em}
\centering
\caption{Variants of the attention bias on building function classification (Singapore), differing in how they combine distance and topology. All other components of \modelname are fixed. \textbf{Bold} = best.}
\label{tab:bias_ablation}
\scalebox{0.85}{
\begin{tabular}{@{}lccc@{}}
\toprule
Pairwise attention bias & Macro-F1 $\uparrow$ & Weighted-F1 $\uparrow$ & Accuracy $\uparrow$ \\
\midrule
No pairwise bias                          & 70.916 $\pm$ 1.373 & 91.449 $\pm$ 0.296 & 91.970 $\pm$ 0.261 \\
Distance only                             & 70.803 $\pm$ 2.416 & 91.750 $\pm$ 0.569 & 92.308 $\pm$ 0.455 \\
Topology only                             & 72.730 $\pm$ 1.408 & 92.379 $\pm$ 0.340 & 92.754 $\pm$ 0.296 \\
Independent distance $+$ topology         & 73.676 $\pm$ 1.370 & 92.591 $\pm$ 0.355 & 92.982 $\pm$ 0.314 \\
Topology-modulated distance (\textbf{\modelname})  & \textbf{74.298 $\pm$ 1.153} & \textbf{92.764 $\pm$ 0.306} & \textbf{93.134 $\pm$ 0.296} \\
\bottomrule
\end{tabular}}
\vspace{-1em}
\end{table}

\end{document}